\documentclass{article} 

\usepackage{microtype}
\usepackage{graphicx}
\usepackage{subfig}
\usepackage{booktabs} 

\usepackage{hyperref}

\usepackage[accepted]{icml2019}
\icmltitlerunning{Learning with Random Learning Rates}

\usepackage{amsfonts}
\usepackage{amsmath}
\usepackage{amssymb}
\usepackage{amsthm}
\usepackage{enumitem}

\usepackage{booktabs}
\usepackage{bbm}
\usepackage{breqn}

\usepackage{mycommands}

\usepackage{listings}
\usepackage{color}

\definecolor{mygreen}{rgb}{0,0.6,0}
\definecolor{mygray}{rgb}{0.5,0.5,0.5}
\definecolor{mymauve}{rgb}{0.58,0,0.82}

\PassOptionsToPackage{hyphens}{url}\usepackage{hyperref}
\usepackage{url}

\newcommand{\deq}{\mathrel{\mathop{:}}=}

\begin{document}

\twocolumn[
\icmltitle{Learning with Random Learning Rates}

\icmlsetsymbol{equal}{*}

\begin{icmlauthorlist}
	
  \icmlauthor{L\'eonard Blier}{equal,tau,fair}
  \icmlauthor{Pierre Wolinski}{equal,tau}
  \icmlauthor{Yann Ollivier}{fair}
\end{icmlauthorlist}

\icmlaffiliation{tau}{TAckling the Underspecified, Universit\'e Paris Sud}
\icmlaffiliation{fair}{Facebook Artificial Intelligence Research}
%

\icmlcorrespondingauthor{L\'eonard Blier}{leonardb@fb.com}
\icmlcorrespondingauthor{Pierre Wolinski}{pierre.wolinski@u-psud.fr}

\icmlkeywords{TODO}
\vskip 0.3in
]

\printAffiliationsAndNotice{\icmlEqualContribution}

\begin{abstract}

In neural networks, the learning rate of the gradient descent 
strongly affects performance. This prevents
reliable out-of-the-box training of a model on a new problem. We propose the \emph{All Learning Rates At
Once} (Alrao) algorithm: each unit or feature in the network gets its own
learning rate sampled from a random distribution spanning several orders
of magnitude, in the hope that enough units will get a close-to-optimal
learning rate. Perhaps surprisingly, stochastic gradient descent (SGD)
with Alrao performs close to SGD with an optimally tuned learning rate,
for various network architectures and problems.
In our experiments, all Alrao
runs were able to learn well without any tuning.

\end{abstract}


\section{Introduction}

Deep learning models often require delicate hyperparameter tuning
\citep{zoph2016neural}: when facing new data or new
model architectures, finding a configuration that makes a model learn can
require both expert knowledge and extensive testing. These and other
issues largely prevent deep learning models from working out-of-the-box
on new problems, or on a wide range of problems, without human
intervention (AutoML setup, \citealt{guyon2016brief}).  One of the most
critical hyperparameters is the learning rate of the gradient descent
\citep[p.~892]{theodoridis2015machine}.  With too large learning rates,
the model does not learn; with too small learning rates, optimization is
slow and can lead to local minima and poor generalization
\citep{jastrzkebski2017three,Kurita2018,Mack2016,Surmenok2017}.

Efficient methods with no learning rate tuning
would be one step towards more robust learning algorithms,
ideally working out-of-the-box. Over the years, many works have tried to
directly set optimal per-parameter learning rates, often inspired
from a second-order, arguably asymptotically optimal analysis using the
Hessian matrix \citep{lecun1998efficient}, or the Fisher information matrix
\citep{Amari98} based on squared gradients. The latter are a key
ingredient in the popular Adam \citep{Kingma2015} optimizer.

Popular optimizers like Adam \citep{Kingma2015} come with
default hyperparameters that reach good performance on many problems and
architectures. Yet fine-tuning and scheduling of the Adam learning rate is
still frequently needed \citep{denkowski2017stronger}, and we suspect the default
setting might be somewhat specific to current problems and architecture
sizes.  Indeed we have found Adam with its default hyperparameters to be
somewhat unreliable over a variety of setups. This would make it unfit in an
out-of-the-box scenario if the right hyperparameters cannot be predicted
in advance.

We propose \emph{All Learning Rates At Once} (Alrao), a gradient descent
method for deep learning models. Alrao uses multiple learning rates at
the same time in the same network, spread across several orders of magnitude.
This creates a mixture of slow and fast learning units, with little added computational burden. 

Alrao departs from the usual philosophy of trying to find the ``right''
learning rates; instead we leverage the redundancy of network-based models
to produce a diversity of behaviors from which good network outputs can be
built.
However, ``wasting'' networks units with
unsuited learning rates might be a concern, a priori resulting in fewer
useful units; so we tested Alrao both with or without increasing network
size. Surprisingly, performance was largely satisfying even without increasing
size.

Overall, Alrao's  performance was always close to that of SGD with the
optimal learning rate. Importantly, 
Alrao was found to combine performance with
\emph{robustness}: not a single run failed to learn,
provided a large enough range of admissible learning rates are included.
In contrast, Adam with its default hyperparameters sometimes just fails
to learn at all, and often exhibits instabilities over the course of
learning even when its peak performance is good.

Thus, in our experiments, we will try to focus not just on performance
(which for an SGD algorithm without learning rate tuning, should ideally
be close to that of optimally-tuned SGD), but also on reliability or
robustness, both during the course of optimization and across different
problems and architectures. 




\paragraph{Contributions.}
\begin{itemize}
\item We introduce Alrao, a gradient descent method with
close-to-optimal performance without learning rate tuning. 
Alrao is found to be reliable over a range of problems and
architectures including convolutional networks, LSTMs, or reinforcement
learning.
\item We compare Alrao to the current default optimizer, Adam
with its default hyperparameters. While Adam sometimes outperforms Alrao, 
it is not reliable across the board when varying architectures or during
training.
\end{itemize}


\section{Related Work}
\label{sec:related-works}
Automatically using the ``right'' learning rate for each parameter was
one motivation behind ``adaptive'' methods such as RMSProp \citep{tieleman2012lecture},
AdaGrad \citep{adagrad} or Adam \citep{Kingma2015}.
Adam with its default setting is currently considered the default method in many
works \citep{wilson2017marginal}, and we use it as a baseline. However, further global
adjustement of the Adam learning rate is common \citep{liu2017progressive}.

Other heuristics for setting the learning rate have been proposed, e.g.,
\citep{pesky}; these heuristics usually start with the idea of
approximating a second-order Newton step to define an optimal learning
rate \citep{lecun1998efficient}. Indeed, asymptotically, an arguably optimal
preconditioner is either
the Hessian of the loss (Newton method) or the Fisher information matrix
\citep{Amari98}.

Such methods directly set per-direction learning rates,
equivalent to
preconditioning the gradient descent with a (diagonal or non-diagonal)
matrix.
From this
viewpoint, Alrao just replaces these preconditioners with a random
diagonal matrix whose entries span several orders of magnitude.

Another approach to optimize the learning rate is to perform a gradient
descent on the learning rate itself through the whole training procedure
(for instance \citep{maclaurin2015gradient}). This can be applied online
to avoid backpropagating through multiple training rounds
\citep{masse2015speed}.  This idea has a long history, see, e.g.,
\citep{schraudolph1999local,mahmood2012tuning}.
Some training algorithms depart from gradient descent
altogether, and become learning rate-free, such as
\citep{orabona2017training} using betting strategies to simulate gradient
descent.

The learning rate can also be optimized within the
framework of architecture search, exploring
both the architecture and learning rate
at the same time (e.g., \citep{real2017large}).
The methods range from reinforcement learning
\citep{zoph2016neural,baker2016designing,
li2017hyperband},
evolutionary algorithms (e.g.,
\citep{stanley2002evolving, jozefowicz2015empirical, real2017large}),  Bayesian optimization
\citep{bergstra2013making} or differentiable architecture search
\citep{liu2018darts}.
These methods are resource-intensive and do
not allow for finding a good learning rate in a single run.

\section{Motivation}
\label{sec:motivation}
Alrao was inspired by the intuition that not all units in a neural
network end up being useful.
Hopefully, in a large enough network, a sub-network made
of units with a good learning rate could learn well, and hopefully the units with a wrong
learning rate will just be ignored. (Units with a too large
learning rate may produce large activation values, so this assumes the
model has some form of protection against those, such as BatchNorm or
sigmoid/tanh activations.)

Several lines of work support the idea that not all units of a network are useful or
need to be trained.
First, it is possible to \emph{prune} a
trained network without reducing the performance too much
(e.g., \citealt{lecun1990, Han2015,Han2015a, See}). Second, training only some of the weights in a neural network
while leaving the others at their initial values performs reasonably
well (see experiments in Appendix~\ref{sec:alrao-bernouilli}).
So in Alrao, units with a very small learning rate
should not hinder training.
\cite{Li2018} even show that performance is reasonable if learning
only within a very small-dimensional affine subspace of the parameters,
\emph{chosen in advance at random} rather than post-selected.

Alrao is consistent with the \emph{lottery ticket hypothesis}, which
posits that ``large networks that train successfully contain subnetworks
that---when trained in isolation---converge in a comparable number of
iterations to comparable accuracy'' \citep{Frankle2018}.  This subnetwork
is the \emph{lottery ticket winner}: the one which had the best initial
values. Arguably, given the combinatorial number of subnetworks in a
large network, with high probability one of them is able to learn alone,
and will make the whole network converge.  Viewing the per-feature
learning rates of Alrao as part of the initialization, this hypothesis
suggests there might be enough sub-networks whose initialization leads to
good convergence.

Alrao specifically exploits the network-type
structure of deep learning models, with their potential excess of
parameters compared to more traditional, lower-dimensional optimization.
That Alrao works at all might already be informative about some phenomena
at play in deep neural networks, relying on the overall network
approach of combining a large number of features built for diversity of
behavior.

\section{All Learning Rates At Once: Description}
\label{sec:idea}
\label{sec:our-method}

\begin{figure}[t]
  \centering
  \includegraphics[width=\linewidth]{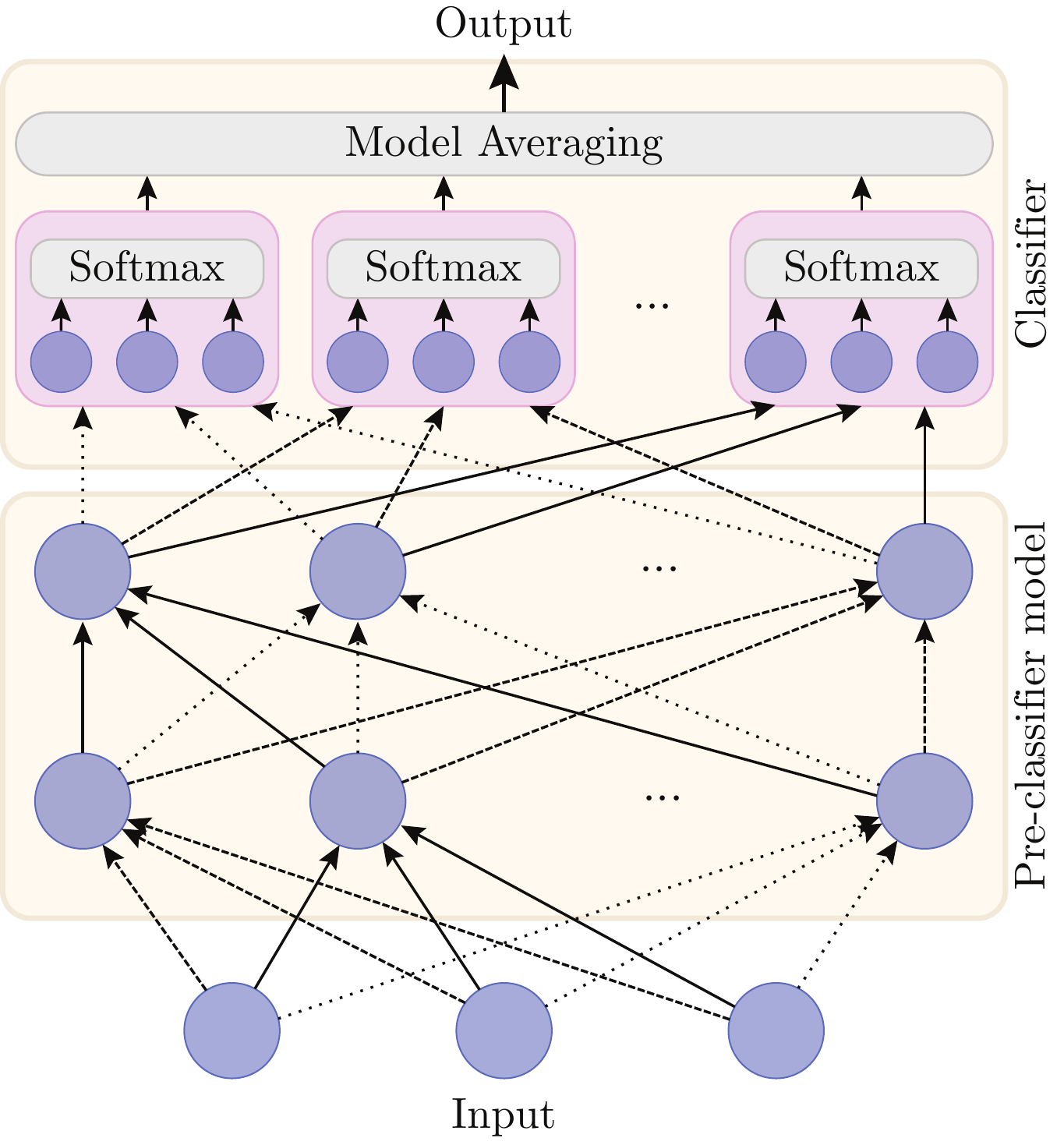}
  \caption{Alrao version of a standard fully connected network for a classification task with three classes.
  The classifier layer is replaced with a set
  of parallel copies of the original classifier, averaged with a model
  averaging method. Each unit uses its own learning rate for its incoming
  weights (represented by the  styles of the arrows).}
\label{fig:archi}
\end{figure}

\paragraph{Principle.} Alrao starts with a standard optimization
method such as SGD, and a range of possible learning rates
$(\eta_{\min}, \eta_{\max})$.
Instead of using a single learning rate, we sample once and for all
one learning rate for each \emph{feature}, randomly sampled log-uniformly
in $(\eta_{\min}, \eta_{\max})$. Then these learning
rates are used in the usual optimization update:
\begin{equation}
  \label{eq:alraoprinciple}
  \theta_{l,i} \leftarrow \theta_{l,i} - \eta_{l,i} \cdot \nabla_{\theta_{l,i}}\ell(\Phi_\theta(x), y)
\end{equation}
where $\theta_{l,i}$ is the set of parameters used to compute the feature
$i$ of layer $l$ from the activations of layer $l-1$ (the \emph{incoming}
weights of feature $i$).
Thus we build ``slow-learning'' and ``fast-learning'' features, in the
hope to get enough features in the ``Goldilocks zone''.

What constitutes a \emph{feature} depends on the type of layers in the
model. For example, in a fully connected layer, each
component of a layer is considered as a feature: all incoming weights of the same unit
share the same learning rate. On the other hand, in a convolutional layer
we consider each convolution filter as constituting a feature: there
is one learning rate per filter (or channel), thus keeping
translation-invariance over the input image. In LSTMs, we apply the same
learning rate to all components in each LSTM unit (thus in
the implementation, the vector of learning rates is the same for input
gates, for forget gates, etc.).

However, the update \eqref{eq:alraoprinciple} cannot be used directly in
the last layer. For instance, for regression there may be only one output
feature. For classification, each feature in the final classification
layer represents a single category,  and so using different learning
rates for these features would favor some categories during learning.
Instead, on the output layer we chose to duplicate the layer using
several learning rate values, and use a (Bayesian) model averaging method
to obtain the overall network output (Fig.~\ref{fig:archi}).
Appendix~\ref{sec:convergence-results}
contains a proof (under convexity assumptions) that this mechanism
works, given the initial layers.

We set a learning rate \emph{per feature}, rather than per
parameter. Otherwise, every feature would have some parameters
with large learning rates, and we would expect even a few large incoming
weights to be able to derail a feature. So having diverging
parameters within a feature is hurtful, while having diverging
features in a layer is not necessarily hurtful since the next layer can
choose to disregard them.  

\paragraph{Definitions and notation.}
\label{sec:notations}
We now describe Alrao more precisely
for deep learning models with softmax output, on
classification tasks (the case of regression
is similar).

Let $\mathcal{D} = \{(x_{1}, y_{1}), ..., (x_{N}, y_{N})\}$, with $y_{i}
\in \{1, ..., K\}$, be a classification dataset.
The goal is to predict the $y_{i}$ given the $x_{i}$, using a deep
learning model $\Phi_{\theta}$. For each input $x$, $\Phi_{\theta}(x)$ is a probability
distribution over $\{1, ..., K\}$, and we want to minimize the categorical cross-entropy loss $\ell$ over the dataset: $\frac{1}{N}\sum_{i}\ell(\Phi_{\theta}(x_{i}), y_{i})$.

A deep learning model for classification $\Phi_{\theta}$ is made of
two parts: a \emph{pre-classifier} $\phi_{\theta^{\mathrm{pc}}}$ which
computes some quantities fed to a
final
\emph{classifier layer} $C_{\theta^{\mathrm{cl}}}$,
namely,
$\Phi_{\theta}(x)=C_{\theta^{\mathrm{cl}}}(\phi_{\theta^{\mathrm{pc}}}(x))$. The
classifier layer $C_{\theta^{\mathrm{cl}}}$ with $K$ categories is
defined by $C_{\theta^{\mathrm{cl}}} = \mathrm{softmax}\circ\left(W^{T}x +
b\right)$ with $\theta^{\mathrm{cl}} = (W, b)$, and $\mathrm{softmax}(x_{1}, ...,
x_{K})_{k} = {e^{x_{k}}}/\left({\sum_{i} e^{x_{i}}}\right).$The
\emph{pre-classifier} is a computational graph composed of any number of \emph{layers},
and each layer is made of multiple \emph{features}.

We denote $\logunif(\cdot ; \eta_{\min}, \eta_{\max})$ the \emph{log-uniform}
probability distribution on an interval $(\eta_{\min}, \eta_{\max})$:
namely, if $\eta \sim \logunif(\cdot ; \eta_{\min},
\eta_{\max})$, then $\log \eta$ is uniformly distributed between $\log
\eta_{\min}$ and $\log \eta_{\max}$.
Its
density function is
\begin{equation}
  \label{eq:logunif}
  \logunif(\eta; \eta_{\min}, \eta_{\max}) = \frac{\mathbbm{1}_{\eta_{\min} \leq \eta \leq \eta_{\max}}}{\log(\eta_{\max}) - \log(\eta_{\min})}\times\frac{1}{\eta}
\end{equation}

\renewcommand{\algorithmiccomment}[1]{// #1}                                       
\begin{algorithm}[t!]
  \caption{Alrao for model $\Phi_{\theta} =
  C_{\theta^{\mathrm{cl}}}\circ\phi_{\theta^{\mathrm{pc}}}$ with
  $N_{\mathrm{cl}}$ classifiers and learning rates in $[\eta_{\min}, \eta_{\max}]$}
  \label{algo:alrao}
  \begin{algorithmic}[1]
    \STATE $a_{j} \leftarrow 1/N_{\mathrm{cl}}$ for each $1\leq j \leq N_{\mathrm{cl}}$ \COMMENT{Initialize the $N_{\mathrm{cl}}$ model averaging weights $a_{j}$}

    \STATE $\Phi^{\text{Alrao}}_{\theta}(x) :=
    \sum_{j=1}^{N_{\mathrm{cl}}}a_{j}\,C_{\theta^{\mathrm{cl}}_{j}}(
    \phi_{\theta^{\mathrm{pc}}}(x))$ \COMMENT{Define the Alrao architecture}
    \FOR{layers $l$, \textbf{for all} feature $i$ in layer $l$}
    \STATE Sample  $\eta_{l,i} \sim \logunif(.; \eta_{\min}, \eta_{\max})$.
    \COMMENT{Sample a learning rate for each feature}
    \ENDFOR
    \FOR{classifiers $j$, $1\leq j \leq N_{\mathrm{cl}}$}
    \STATE Define  $\log \eta_{j} = \log\eta_{\min} + \frac{j-1}{N_{\mathrm{cl}}-1}\log \frac{\eta_{\max}}{\eta_{\min}}$.
    \COMMENT{Set a learning rate for each classifier $j$}
    \ENDFOR

    \WHILE{stopping criterion is false}
        \STATE $z_{t} \leftarrow \phi_{\theta^{\mathrm{pc}}}(x_{t})$
	\COMMENT{Store the pre-classifier output}
        \FOR{layers $l$, \textbf{for all} feature $i$ in layer $l$}
            \STATE $\theta_{l,i} \leftarrow \theta_{l,i} - \eta_{l,i} \cdot \nabla_{\theta_{l,i}}\ell(\Phi^{\text{Alrao}}_{\theta}(x_{t}), y_{t})$ \COMMENT{Update the pre-classifier weights}
        \ENDFOR
        \FOR{Classifier $j$}
            \STATE $\theta^{\mathrm{cl}}_{j} \leftarrow
	    \theta^{\mathrm{cl}}_{j} - \eta_{j}\cdot
	    \nabla_{\theta^{\mathrm{cl}}_{j}} \,
	    \ell(C_{\theta^{\mathrm{cl}}_{j}}(z_{t}), y_{t})$
	    \COMMENT{Update the classifiers' weights}
        \ENDFOR

        \STATE $a \leftarrow \mathtt{ModelAveraging}(a, (C_{\theta^{\mathrm{cl}}_{i}}(z_{t}))_{i}, y_{t})$ \COMMENT{Update the model averaging weights.}
        \STATE $t \leftarrow t+1$ mod $N$
    \ENDWHILE
    \end{algorithmic}
  \end{algorithm}

\paragraph{Alrao for the pre-classifier: A random learning rate
for each feature.}
In the pre-classifier, for each feature $i$ in each layer $l$, a learning
rate $\eta_{l,i}$ is sampled from the probability distribution
$\logunif(.; \eta_{\min}, \eta_{\max})$, once and for all at the beginning of
training.\footnote{With learning rates resampled at each time,
each step would be, in expectation, an ordinary SGD step with
learning rate $\mathbb{E}\eta_{l,i}$, thus just
yielding an ordinary SGD trajectory with more noise.} Then the incoming
parameters of each feature in the preclassifier are updated
in the usual way
with this learning rate (Eq.~\ref{eq:updatepc}).

\paragraph{Alrao for the classifier layer: Model averaging from
classifiers with different learning rates.}
\label{sec:parall-class}
In the classifier layer,
we build multiple clones of the original classifier layer, set a
different learning
rate for each, and then use a model averaging method from among them. The
averaged classifier and the overall Alrao model are:
\begin{align}
  \label{eq:parall-class}
  C^{\text{Alrao}}_{\theta^{\mathrm{cl}}}(z) \deq
  \sum_{j=1}^{N_{\mathrm{cl}}}a_{j} \, C_{\theta^{\mathrm{cl}}_{j}}(z) \\
\Phi^{\text{Alrao}}_{\theta}(x) \deq
C^{\text{Alrao}}_{\theta^{\mathrm{cl}}}(\phi_{\theta^{\mathrm{pc}}}(x))
\end{align}
where the
$C_{\theta^{\mathrm{cl}}_{j}}$ are copies of the original classifier
layer, with
non-tied parameters, and
$\theta^{\mathrm{cl}} \deq (\theta^{\mathrm{cl}}_{1}, ...,
\theta^{\mathrm{cl}}_{N_{\mathrm{cl}}})$. The $a_{j}$ are the parameters of the model
averaging, and are such that for all $j$, $0 \leq a_{j} \leq 1$, and
$\sum_{j}a_{j} = 1$. These are not updated by gradient descent, but via a
model averaging method from the literature (see below).

For each classifier $C_{\theta^{\mathrm{cl}}_{j}}$, we set a learning
rate $\eta_{j}$ defined by
$
\log \eta_{j} = \log \eta_{\min} +
\frac{j-1}{N_{\mathrm{cl}}-1}\log\left(\frac{\eta_{\max}}{\eta_{\min}}\right)
$, 
so that the
classifiers' learning rates are log-uniformly spread on the interval
$[\eta_{\min}, \eta_{\max}]$.

Thus, the original model $\Phi_{\theta}(x)$ leads to the Alrao model $\Phi^{\text{Alrao}}_{\theta}(x)$.
Only the classifier layer is modified, the pre-classifier architecture being unchanged.

\paragraph{Update rule.}

Alg.~\ref{algo:alrao} presents the full Alrao algorithm.
The updates for the pre-classifier, classifier, and model
averaging weights are as follows.
\begin{itemize}
 \item The update rule for the pre-classifier is the usual SGD one, with
per-feature learning rates. For each feature $i$ in each
layer $l$, its incoming parameters are updated as:
  \begin{equation}
    \label{eq:updatepc}
  \theta_{l,i} \leftarrow \theta_{l,i} - \eta_{l,i} \cdot \nabla_{\theta_{l,i}}\ell(\Phi^{\text{Alrao}}_\theta(x), y)
\end{equation}

\item The parameters $\theta^{\mathrm{cl}}_j$ of each classifier clone $j$ on the classifier layer are
updated
as if this classifier alone was the only output of the model:
\begin{align}
    \label{eq:updatec}
  \theta^{\mathrm{cl}}_{j} \leftarrow & \;
  \theta^{\mathrm{cl}}_{j}  - \eta_{j} \cdot
  \nabla_{\theta^{\mathrm{cl}}_{j}}\,\ell(C_{\theta^{\mathrm{cl}}_{j}}(\phi_{\theta^{\mathrm{pc}}}(x)), y)
\end{align}
(still sharing the same pre-classifier $\phi_{\theta^{\mathrm{pc}}}$).
This
ensures classifiers with low
weights $a_j$ still learn, and
is consistent with model averaging philosophy.
Algorithmically this requires differentiating the loss $N_{\mathrm{cl}}$
times with respect to the last layer (but no additional backpropagations
through
the preclassifier).

\item To set the weights $a_j$, several model averaging techniques
are available, such as Bayesian Model
Averaging \citep{Wasserman2000}. We decided to use the
\emph{Switch} model averaging \citep{VanErven2011}, a Bayesian method which is
both simple,
principled and
very responsive to changes in performance of
the various models. After each sample or mini-batch, the switch computes
a modified posterior distribution $(a_j)$ over the classifiers. This computation
is directly taken from \citep{VanErven2011} and explained in
Appendix~\ref{sec:switch}. The observed evolution of this posterior
during training is commented upon in Appendix~\ref{sec:posterior}.
\end{itemize}

\begin{table*}[t!]
   \caption{
Performance of Alrao, of SGD with optimal learning rate from
$\{10^{-5}, 10^{-4}, 10^{-3}, 10^{-2}, 10^{-1}, 1., 10.\}$, and of Adam
with its default setting. Three convolutional models are reported for image
classification on CIFAR10, three others for ImageNet, one recurrent model for character prediction (Penn
Treebank), and two experiments on RL problems. The Alrao learning rates have been taken in wide a priori
reasonable intervals, $[\eta_{\min};\eta_{\max}] = [10^{-5};10]$ for CNNs
(CIFAR10 and ImageNet) and RL, and $[10^{-3};10^2]$ for RNNs (PTB).
Each experiment is
run 10 times (CIFAR10 and RL), 5 times (PTB) or 1 time (ImageNet); the confidence intervals
report the standard
deviation over these runs. %
     }
  \begin{center}
  \fontsize{9pt}{9pt}\selectfont
  \begin{sc}
  \begin{tabular}[h]{lcccccccc}
    \toprule
    Model              & \multicolumn{3}{c}{SGD with optimal LR} & \multicolumn{2}{c}{Adam - Default} & \multicolumn{2}{c}{Alrao}                                                                      \\
    \cmidrule(lr){2-4} \cmidrule(lr){5-6} \cmidrule(lr){7-8}
    {}                 & LR                                      & Loss                               & Top1 (\%)      & Loss             & Top1  (\%)     & Loss              & Top1  (\%)            \\
    \midrule
    \multicolumn{1}{c}{\emph{CIFAR10}}                                                                                                                                                                 \\
    MobileNet          & $1e$-1                                  & $0.37 \pm 0.01$                    & $90.2 \pm 0.3$ & $1.01 \pm 0.95$  & $78 \pm 11$    & $0.42 \pm 0.02$   & $88.1 \pm 0.6$        \\
    MobileNet, width*3 &  -                                       &   -                                 &    -            & $0.32 \pm 0.02$  & $90.8 \pm 0.4$ & $0.35 \pm 0.01$   & $89.0 \pm 0.6$ \\
    GoogLeNet          & $1e$-2                                  & $0.45 \pm 0.05$                    & $89.6 \pm 1.0$ & $0.47 \pm 0.04$  & $89.8 \pm 0.4$ & $0.47 \pm 0.03$   & $88.9 \pm 0.8$        \\
    GoogLeNet, width*3 &  -                                       &   -                                 &    -            & $0.41 \pm 0.02$  & $88.6 \pm 0.6$ & $0.37 \pm 0.01$   & $89.8 \pm 0.8$ \\
    VGG19              & $1e$-1                                  & $0.42 \pm 0.02$                    & $89.5 \pm 0.2$ & $0.43 \pm 0.02 $ & $88.9 \pm 0.4$ & $0.45 \pm 0.03$   & $87.5 \pm 0.4$        \\
    VGG19, width*3     &  -                                       &     -                               &   -             & $0.37 \pm 0.01$  & $89.5 \pm 0.8$ & $0.381 \pm 0.004$ & $88.4 \pm 0.7$ \\

    \midrule
    \multicolumn{1}{c}{\emph{ImageNet}}                                                                                                                                                                         \\
    AlexNet                                           & $1e$-2 & $2.15$                             & $53.2$                            & $6.91$            & $0.10 $        & $2.56$          & $43.2$         \\
    Densenet121                                       & $1$    & $1.35 $                            & $69.7 $                           & $1.39 $           & $67.9 $        & $1.41 $         & $67.3  $       \\
    ResNet50                                          & $1$    & $1.49$                             & $67.4 $                           & $1.39 $           & $67.1 $        & $1.42$          & $67.5$         \\
    ResNet50, width*3                                 & -     & -                                 & -                               & $1.99$            & $60.8$         & $1.33$          & $70.9$         \\
    \midrule
    \multicolumn{1}{c}{\emph{Penn Treebank}}                                                                                                                                                                    \\
    LSTM                                              & $1$    & $1.566 \pm 0.003$                  & $66.1 \pm 0.1$                    & $1.587 \pm 0.005$ & $65.6 \pm 0.1$ & $1.67 \pm 0.01$ & $64.1 \pm 0.2$ \\
    \midrule
    \multicolumn{1}{c}{\emph{Reinforcement Learning}} &        & \multicolumn{2}{c}{Return}         & \multicolumn{2}{c}{Return}        & \multicolumn{2}{c}{Return}                                            \\
    Pendulum                                          & $1e-4$ & \multicolumn{2}{c}{$-372 \pm 24$}  & \multicolumn{2}{c}{$-414 \pm 64$} & \multicolumn{2}{c}{$-371 \pm 36$}                                     \\
    LunarLander                                       & $1e-1$ & \multicolumn{2}{c}{$188 \pm  23$ } & \multicolumn{2}{c}{$155 \pm 23$}  & \multicolumn{2}{c}{$186 \pm 45$}                                      \\
    \bottomrule
  \end{tabular}
  \end{sc}
  \end{center}

  \vskip -0.1in
  \label{tab:results}
\end{table*}

\paragraph{Implementation.} \label{sec:implementation} We release along
with this paper a Pytorch \citep{paszke2017automatic} implementation of
this method. It can be used on an existing model with little
modification. A short tutorial is given in Appendix~\ref{sec:tutorial}.
\emph{Features} (sets of weights sharing the same learning rate) need
to be specified for each layer type: for now this has been done for linear,
convolutional, and LSTMs layers.

\section{Experimental Setup}
\label{sec:experiments}

We tested Alrao on various convolutional networks for image
classification (Imagenet and CIFAR10), on LSTMs for text prediction, and on Reinforcement Learning problems. The
baselines are SGD with an optimal learning rate, and Adam with its default
setting, arguably the current default method \citep{wilson2017marginal}.

\paragraph{Image classification on ImageNet and CIFAR10.}
\label{sec:cifar10}
For image classification, we used the
ImageNet \citep{imagenet_cvpr09} and CIFAR10 \citep{Krizhevsky2009}
datasets.
The ImageNet dataset is made of 1,283,166 training and 60,000 testing
data; we split the training set into a smaller training set and a
validation set with 60,000 samples. We do the same on CIFAR10: the 50,000
training samples are split into 40,000 training samples and 10,000 validation samples.

For each architecture, training on the smaller training set was
stopped when the validation loss had not improved for 20 epochs. The
epoch with best validation loss was selected and the corresponding model tested on the
test set.
The inputs are normalized, and training used data augmentation: random
cropping and random horizontal flipping (see
Appendix~\ref{sec:preprocess} for details).
For CIFAR10, each setting was run 10 times: the confidence intervals presented are the
standard deviation over these runs. For ImageNet, because of high computation time, we performed only a single run per experiment.

We tested  Alrao on
several standard architectures for these
tasks. On ImageNet, we tested Resnet50~\citep{he2016deep},
Densenet121~\citep{huang2017densely} and
Alexnet~\citep{krizhevsky2014one}, with the default Pytorch implementation. On CIFAR10, we tested
GoogLeNet~\citep{szegedy2015going}, VGG19~\citep{Simonyan14c} and
MobileNet \citep{howard2017mobilenets} implemented by~\citep{kiangliu}. 

The Alrao learning rates were sampled log-uniformly from
$\eta_{\min} = 10^{-5}$ to $\eta_{\max} = 10$.
For the output layer we used 10 classifiers with switch model averaging
(Appendix~\ref{sec:switch}); the learning rates of the output classifiers
are deterministic and log-uniformly spread in
$[\eta_{\min},\eta_{\max}]$.

In addition, each model was trained with SGD for every learning rate in the set
$\{10^{-5}, 10^{-4}, 10^{-3}, 10^{-2}, 10^{-1}, 1., 10.\}$. The best SGD
learning rate is selected on the validation set, then reported
in Table~\ref{tab:results}. We also compare to Adam with its default hyperparameters ($\eta=10^{-3}, \beta_1 = 0.9, \beta_2 = 0.999$).

Finally, since Alrao may waste units with unsuitable learning rates, we
also tested architectures with increased width ($3$ times as many units)
with Alrao and
Adam on ImageNet and CIFAR10. On these larger models, 
systematic SGD learning rate grid search was not performed due to the time required.

The results are presented in Table~\ref{tab:results}.
Learning curves with various SGD learning rates, Adam and
Alrao are presented in Fig.~\ref{fig:firstepochs}.
Fig.~\ref{fig:trig} tests the influence of $\eta_{\min}$ and
$\eta_{\max}$.

\begin{figure*}[t]
  \centering
    \subfloat[Resnet50 on ImageNet\label{fig:firstepochs-resnet}.]{
      \includegraphics[width=0.95\linewidth]{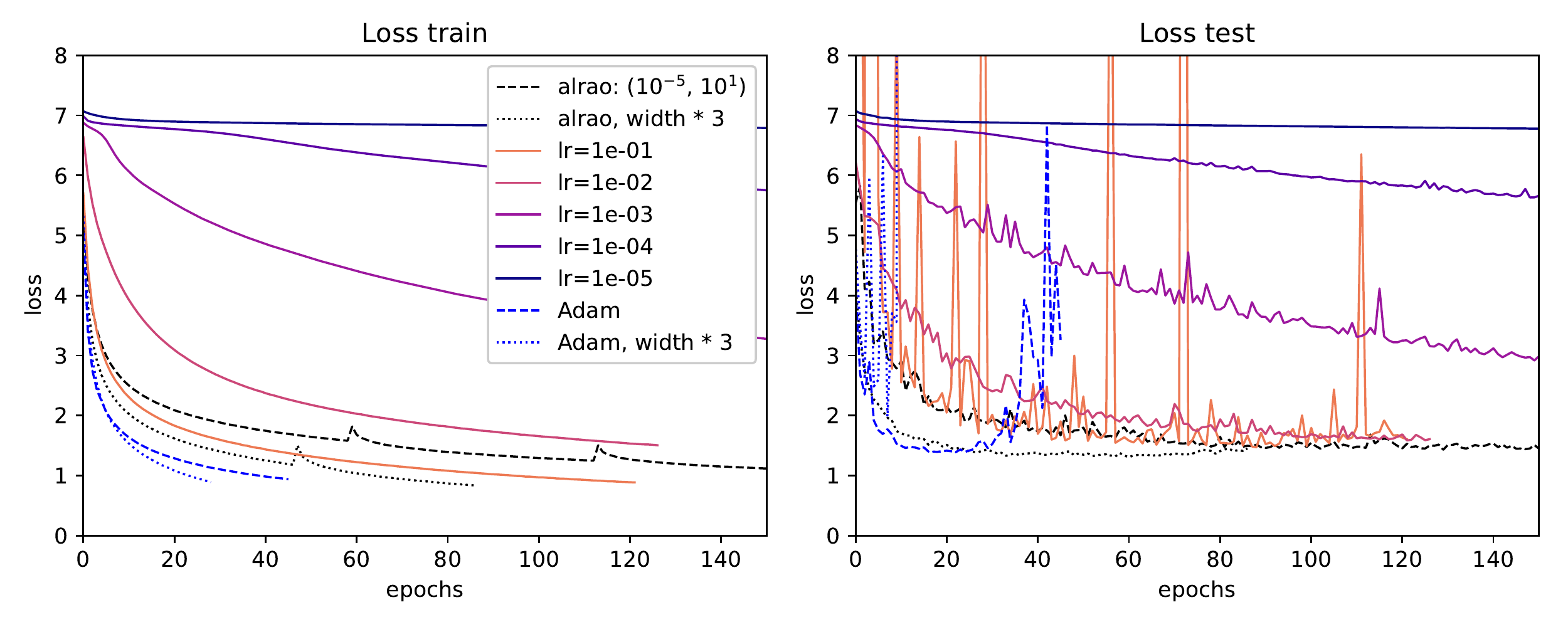}
    }
    \vfill
    \subfloat[AlexNet on ImageNet\label{fig:firstepochs-alexnet}]{
        \includegraphics[width=0.95\linewidth]{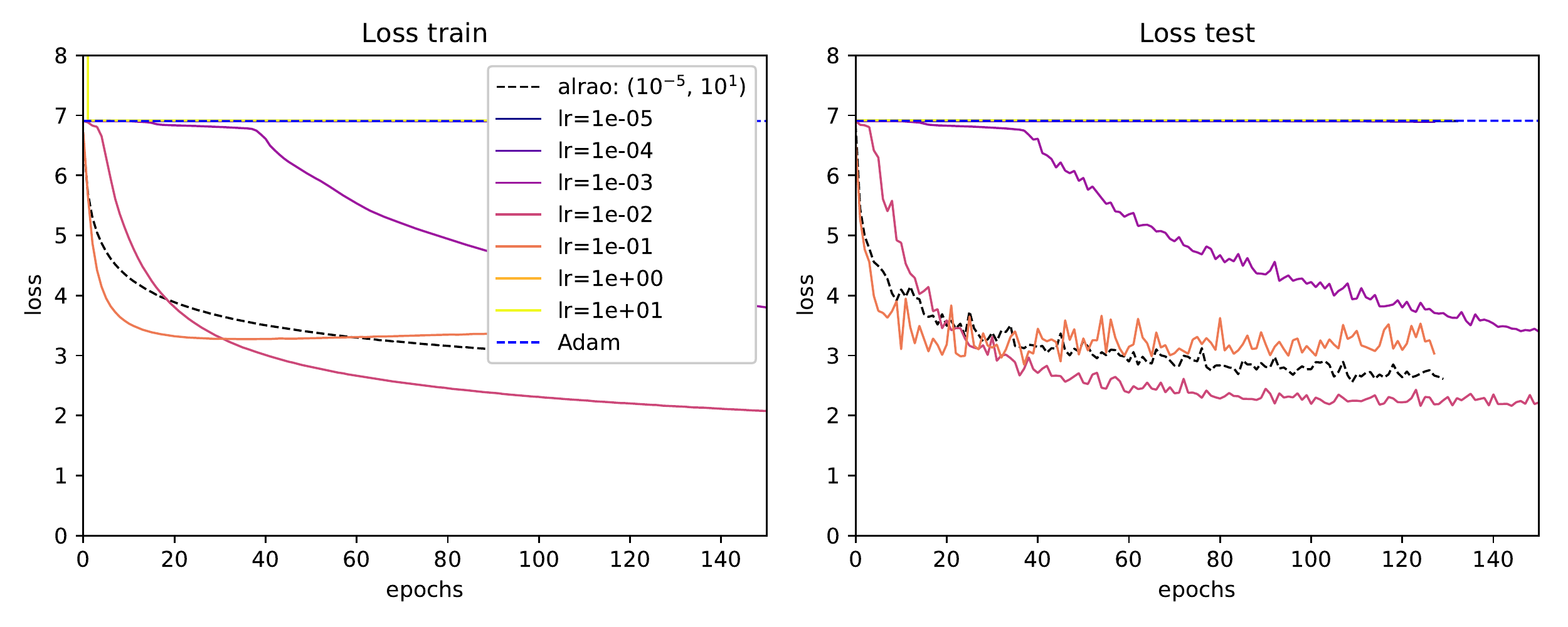}
      }
  \caption{
Learning curves for SGD
with various learning rates, Alrao, and Adam with its default
setting, with the Resnet50 architecture on ImageNet and the MobileNetV2 architecture on CIFAR10. Left: training loss; right: test
loss. While Alrao uses learning rates from the entire
range, its performance is comparable to the
optimal learning rate.}
   \label{fig:firstepochs}
\end{figure*}

\paragraph{Recurrent learning on Penn Treebank.}
\label{sec:penn-tree-bank}
To test Alrao on a different kind of architecture,
we used a recurrent neural network for text prediction on the Penn
Treebank \citep{Marcus1993} dataset.
The experimental procedure is the same, with $(\eta_{\min}, \eta_{\max}) =
(0.001, 100)$ and $6$ output classifiers for Alrao. The results appear in
Table~\ref{tab:results}, where the loss is given in bits per character and the accuracy is the proportion of correct character predictions.

The model was trained for \emph{character} prediction rather than word
prediction. This is technically easier for Alrao implementation: since Alrao uses copies
of the output layer, memory issues arise for models with most parameters on
the output layer. Word prediction (10,000 classes on PTB) requires more output
parameters than character prediction; see Section~\ref{sec:discussion}
and Appendix~\ref{sec:number-parameters}.

The model is a two-layer LSTM \citep{hochreiter1997long} with an embedding size of 100 and
100 hidden features. A dropout layer with rate $0.2$ is included before
the decoder. The training set is divided into 20 minibatchs.
Gradients are computed via truncated backprop through time \citep{werbos1990backpropagation} with truncation every 70 characters.

\paragraph{Reinforcement Learning.}
We tested Alrao on two standard Reinforcement Learning problems: the Pendulum and Lunar Lander environments from OpenAI Gym~\cite{gym}. We use standard Deep Q-learning~\cite{dqn}. The $Q$-network is a standard MLP with 2 hidden layers. The experimental setting to compare Alrao, Adam and SGD is the same as above, with $\eta_{\min} = 10^{-5}$ to $\eta_{\max} = 10$. Alrao uses 10 output layers (which are not classifiers in that case but regressors). More details on the Q-learning implementation are given in Appendix~\ref{sec:preprocess}.
For each environment, we selected the best epoch on evaluation runs, and then reported the return of the selected model on new runs in that environment.

\section{Performance and Robustness of Alrao}

\paragraph{Performance of Alrao compared to SGD with the optimal learning rate.}
As expected, Alrao usually performs slightly worse than the best learning
rate with SGD.

Still, even with wide intervals $(\eta_\min, \eta_\max)$, Alrao comes
reasonably close to the best learning rate, across every setup. Notably,
this occurs even though SGD achieves good performance only for a few
learning rates within the interval $(\eta_\min, \eta_\max)$. With our
setting for image classification and RL ($\eta_\min = 10^{-5}$ and
$\eta_\max = 10$), among the $7$ learning rates used with SGD (${10^{-5},
10^{-4}, 10^{-3}, 10^{-2}, 10^{-1}, 1, 10}$), only 3 are able to learn
with AlexNet (and only one is better than Alrao, see Fig.
\ref{fig:firstepochs-alexnet}), only 3 are able to learn with ResNet50
(and only two of them achieve performance similar to Alrao, see Fig.~\ref{fig:firstepochs-resnet}), and only 2 are able to learn on the
pendulum environment (and only one of them converges as fast as Alrao,
Fig. \ref{fig:rl-learningcurves} in Appendix \ref{sec:preprocess}). More
examples are in Appendix \ref{sec:preprocess}. It is surprising that
Alrao manages to learn even though most of the units in the network would
have learning rates unsuited for SGD.

\paragraph{Robustness of Alrao compared to default Adam.} Overall, Alrao learns
reliably in every setup. Performance is close to optimal SGD in all cases
(Table~\ref{tab:results}) with a somewhat larger gap in one case (AlexNet
on ImageNet). This observation is
quite stable over the course of learning, with Alrao curves shadowing
optimal SGD curves over time (Fig.~\ref{fig:firstepochs}).

Often, Adam with its default parameters almost matches optimal SGD, but
this is not always the case.  Over the 13 setups in
Table~\ref{tab:results}, default Adam gives a significantly
poor performance in three cases.  One of those is a pure optimization
issue: with AlexNet on ImageNet, optimization does not start
(Fig.~\ref{fig:firstepochs-alexnet}) with default parameters.
The other two cases are due to strong overfit despite good train
performance: MobileNet on CIFAR (Fig. \ref{fig:firstepochs-mobilenet} in Appendix~\ref{sec:preprocess}) and ResNet with increased width on ImageNet (Fig. \ref{fig:firstepochs-resnet}).

In two further cases, Adam achieves good validation performance but overfits
shortly thereafter: ResNet (Fig.~\ref{fig:firstepochs-resnet})and
DenseNet (Fig.~\ref{fig:firstepochs-densenet} in Appendix~\ref{sec:preprocess}), both on ImageNet).  On the whole, this confirms a known risk of overfit with Adam
\citep{wilson2017marginal}.

Overall, default Adam tends to give slightly better results than Alrao when it
works, but does not learn reliably with its default hyperparameters. 
It can exhibit two kinds of lack of robustness:
optimization failure, and overfit or non-robustness over the course of
learning.  
On
the other hand, every single run of Alrao reached reasonably
close-to-optimal performance.
Alrao also exhibits a steady performance over the
course of learning (Fig.~\ref{fig:firstepochs}).



%


\section{Limitations and Perspectives}
\label{sec:discussion}
\label{sec:strengths-weaknesses}

\begin{figure*}[!h]
  \centering
  \includegraphics[width=\textwidth]{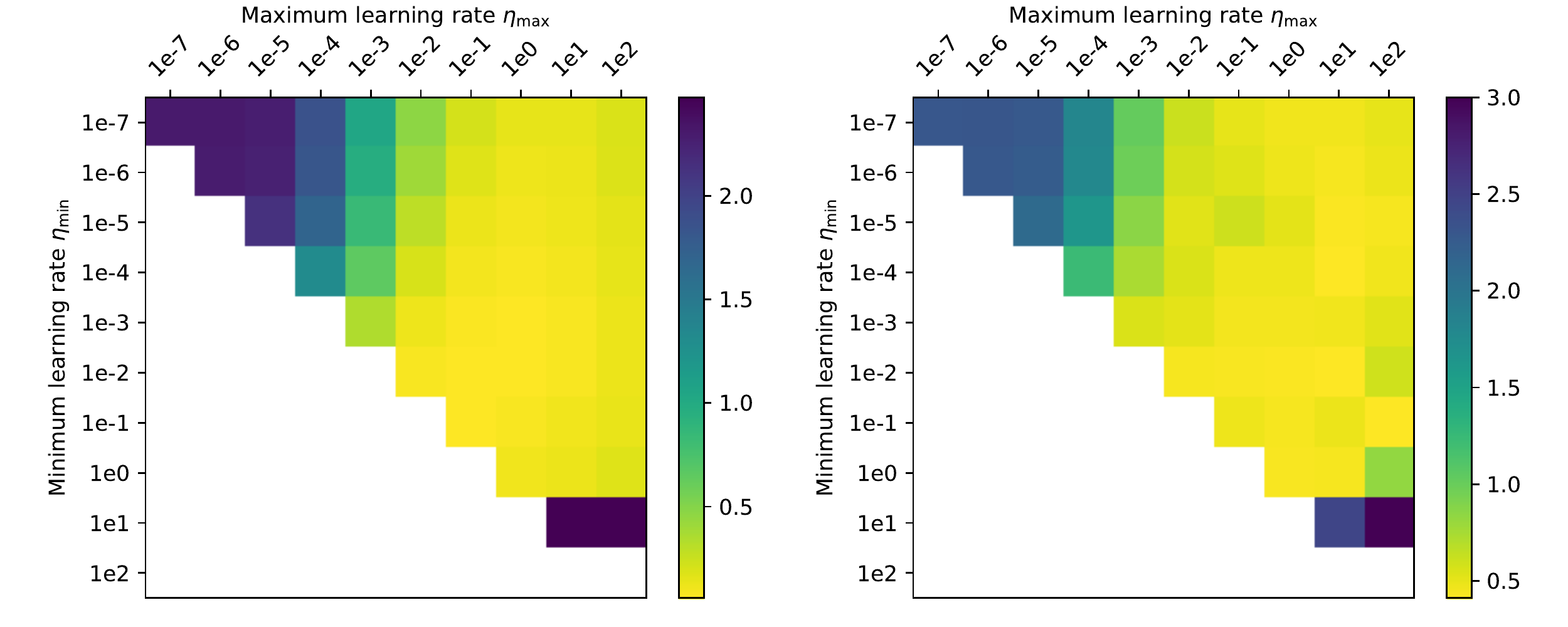}
  \caption{Performance of Alrao with a GoogLeNet model on CIFAR10, depending on the interval
  $(\eta_\min, \eta_\max)$. Left: loss on the train set; right: on the
  test set. Each point with coordinates $(\eta_\min, \eta_\max)$ above
  the diagonal represents the loss after 30 epochs for Alrao with this
  interval. Points $(\eta, \eta)$ on the diagonal represent
  standard SGD
  with learning rate $\eta$ after 50 epochs. Standard SGD with $\eta = 10^2$ is
  left blank to due numerical divergence (NaN). Alrao works as soon as
  $(\eta_\min, \eta_\max)$ contains at least one suitable learning
  rate.}
  \label{fig:trig}
\end{figure*}

\paragraph{Increased number of parameters for the classification layer.}
Alrao modifies the output layer of the optimized model. The number of
parameters for the classification layer is multiplied by the number of
classifier copies used (the number of parameters in the pre-classifier is
unchanged).  On CIFAR10 (10 classes), the number of parameters increased
by less than 5\% for the models used. On ImageNet (1000 classes), it
increases by 50--100\% depending on the architecture. On Penn Treebank, the number of
parameters increased by $15\%$ in our setup (working at the character
level); working at word level it would have increased threefold
(Appendix~\ref{sec:number-parameters}). 
Still, models with a very large number of output classes usually
rely on other parameterizations than a direct softmax, such as a
hierarchical softmax (see
references in \citep{jozefowicz2016exploring});
Alrao can be used in conjunction with such methods.


This would clearly be a limitation of Alrao for models with most parameters in the
classifier layer and without existing methods to streamline this layer.  For
such models, this could be mitigated by handling the copies
of the classifiers on distinct computing units: in Alrao these copies
work in parallel given the pre-classifier.

\paragraph{Adding two hyperparameters.}
We claim to remove a hyperparameter, the learning rate, but
replace it with two hyperparameters $\eta_{\min}$ and $\eta_{\max}$.
Formally, this is true. But a systematic study of the impact of these two
hyperparameters (Fig.~\ref{fig:trig}) shows that the sensitivity to
$\eta_{\min}$ and $\eta_{\max}$ is much lower than the original sensitivity
to the learning rate. In our experiments, convergence happens as soon as
$(\eta_{\min};\eta_{\max})$ contains a reasonable learning rate
(Fig.~\ref{fig:trig}).

A wide range of values of $(\eta_{\min};\eta_{\max})$
will contain one good learning rate and achieve
close-to-optimal performance (Fig.~\ref{fig:trig}).
Typically, we recommend to just use an interval containing
all the learning rates that would have been tested in a grid search,
e.g., $10^{-5}$ to $10$.

So, even if the choice of $\eta_{\min}$ and $\eta_{\max}$ is important, the
results are much more stable to varying these two hyperparameters than
to the original learning rate.
For instance, standard SGD fails due to numerical issues for $\eta =
100$ while Alrao with $\eta_\max = 100$ works with any $\eta_\min\leq 1$
(Fig.~\ref{fig:trig}), and is thus stable to relatively large learning
rates.
We would still expect numerical issues 
with very large $\eta_\max$, but this has not been observed in our experiments.

\label{sec:remarks}
\label{sec:future-directions}

\paragraph{Increasing network size.}
With Alrao, neurons with unsuitable learning rates will not learn:
those with a too large learning rate might learn nothing, while
those with too small learning rates will learn too slowly to be used.
Thus, Alrao may reduce the \emph{effective
size} of the network to only a fraction of the actual architecture size, depending on $(\eta_{\min}, \eta_{\max})$.
Our first intuition was that increasing the width of the network was
going to be necessary with Alrao, to avoid wasting too many units.

Tests of Alrao with increased width are reported in
Table~\ref{tab:results} and  Fig. \ref{fig:firstepochs-resnet}.
Incidentally, these tests show that width was a limiting factor of the
models used for both Alrao and SGD. 
Still, to our surprise, Alrao worked
well even without width augmentation.

\paragraph{Other optimizers, other hyperparameters, learning rate schedulers...}
Using a learning rate schedule instead of a fixed learning rate is
often effective \citep{bengio2012practical}. We did not use learning rate
schedulers here; this may partially explain why the results in Table~\ref{tab:results} are worse than the state-of-the-art.
Nothing prevents using such a scheduler within Alrao, e.g., by dividing
all Alrao learning rates by a time-dependent constant; we did not
experiment with this yet. One
might have hoped that Alrao would match good stepsize schedules thanks to
the diversity of learning rates, but our results do not currently support
this. 

The Alrao idea can also be used with other optimizers than SGD, such as
Adam. We tested combining Alrao and Adam, and found the combination
less reliable than standard Alrao (Appendix~\ref{sec:alrao-adam},
Fig.~\ref{fig:adam-alrao}).  This occurs mostly for
test performance,
while training curves mostly look good (Fig.~\ref{fig:adam-alrao}).
The stark train/test
discrepancy suggests that Alrao combined with Adam may perform well as a pure
optimization method but exacerbates the underlying risk of overfit of
Adam \citep{wilson2017marginal,keskar2017improving}.

The Alrao idea could be used on other hyperparameters as
well, such as momentum. However, if more hyperparameters are initialized
randomly for each feature, the fraction of features having all their hyperparameters in
the ``Goldilocks zone'' will quickly decrease.

\section{Conclusion}
\label{sec:conclusion}

Applying stochastic gradient descent with multiple learning rates for
different features is surprisingly resilient in our experiments, and
provides performance close enough to SGD with an optimal learning rate,
as soon as the range of random learning rates contains a suitable one.
The same resilience is not observed with default Adam.
Alrao could save time when testing deep learning models, opening the door
to more out-of-the-box uses of deep learning.

\section*{Acknowledgments}
We would like to thank Corentin Tallec for his technical help, and his many remarks and advice. We thank Olivier Teytaud for pointing useful references, and Guillaume Charpiat and L\'eon Bottou for their remarks.

\bibliographystyle{plainnat}

\bibliography{biblio}

\vfill

\pagebreak

~
\pagebreak

\appendix

\section{Model Averaging with the Switch}
\label{sec:switch}

As explained is Section~\ref{sec:our-method}, we use a model averaging
method on the classifiers of the output layer. We could have used the
Bayesian Model Averaging method \citep{Wasserman2000}. But one of its
main weaknesses is the \emph{catch-up phenomenon}
\citep{VanErven2011}: plain Bayesian posteriors are slow to react
when the relative performance of models changes over time. Typically,
for instance, some larger-dimensional models need
more training data to reach good performance: at the time they become
better than lower-dimensional models for predicting current data, their
Bayesian posterior is so bad that they are not used right away (their
posterior needs to ``catch up'' on their bad initial performance).
This leads to very conservative model averaging methods.

The solution from \citep{VanErven2011} against the catch-up
phenomenon is to \emph{switch} between models. It is based on previous
methods for
prediction with expert advice (see for instance
\citep{herbster1998tracking,volf1998switching} and the references in
\citep{koolen2008combining,VanErven2011}), and is well rooted in
information theory.
The switch method maintains a
Bayesian posterior distribution, not over the set of models, but over the
set of \emph{switching strategies} between models. Intuitively, the model
selected can be adapted online to the number of samples seen.

We now give a quick overview of the switch method from
\cite{VanErven2011}: this is how the model
averaging weights $a_j$ are chosen in Alrao.

Assume that we have a set
of prediction strategies $\mathcal{M} = \{p^{j}, j \in
\mathcal{I}\}$. We define the set of \emph{switch sequences}, $\BS = \{((t_{1},
j_{1}), ..., (t_{L}, j_{L})), 1 = t_{1} < t_{2} < ... < t_{L}\; , \;  j
\in \mathcal{I}\}$.  Let $s = ((t_{1}, j_{1}), ..., (t_{L}, j_{L}))$ be a
switch sequence. The associated prediction strategy
$p_{s}(y_{1:n}|x_{1:n})$ uses model $p^{j_i}$ on the time interval
$[t_i;t_{i+1})$, namely \begin{align} \label{eq:defswitch}
p_{s}(y_{1:i+1}|x_{1:i+1},y_{1:i}) =
p^{K_{i}}(y_{i+1}|x_{1:i+1},y_{1:i}) \end{align} where $K_{i}$ is such
that $K_{i} = j_{l}$ for $t_{l} \leq i < t_{l+1}$.  We fix a prior
distribution $\pi$ over switching sequences. In this work, ${\cal I} = \{1, ..., N_C\}$ the prior is, for a switch sequence $s = ((t_{1}, j_{1}), ..., (t_{L}, j_{L}))$:
\begin{equation}
  \label{eq:priorswitch}
  \pi(s) = \pi_L(L)\pi_K(j_1)\prod_{i=2}^{L} \pi_T(t_{i}|t_i > t_{i-1})\pi_K(j_i)
\end{equation}
with $\pi_L(L) = \frac{\theta^L}{1-\theta}$ a geometric distribution over the switch sequences lengths, $\pi_K(j) = \frac{1}{N_C}$ the uniform distribution over the models (here the classifiers) and $\pi_T(t) = \frac{1}{t(t+1)}$.

This defines a Bayesian mixture distribution:
\begin{equation}
  \label{eq:2}
  p_{sw}(y_{1:T}|x_{1:T}) = \sum_{s\in\BS}\pi(s)p_{s}(y_{1:T}|x_{1:T})
\end{equation}
Then, the model averaging weight $a_{j}$ for the classifier $j$ after seeing $T$ samples is the posterior of the switch distribution: $\pi(K_{T+1}=j|y_{1:T}, x_{1:T})$.
\begin{align}
  a_{j} &= p_{sw}(K_{T+1}=j|y_{1:T}, x_{1:T}) \\
        &= \frac{p_{sw}(y_{1:T}, K_{T+1}=j | x_{1:T})}{p_{sw}(y_{1:T}| x_{1:T})}
\end{align}
These weights can be computed online exactly in a quick and simple way
\citep{VanErven2011}, thanks to dynamic programming methods from hidden
Markov models.

The implementation of the switch used in Alrao exactly
follows the pseudo-code from \citep{NIPS2007_3277}, with hyperparameter
$\theta = 0.999$ (allowing for many switches a priori). It can be found in the accompanying online code.


\section{Convergence result in a simple case}
\label{sec:convergence-results}

We prove a convergence result on Alrao in a simplified case: we assume
that the loss is convex, that the pre-classifier is fixed, that we work
with full batch gradients rather than stochastic gradient descent, and
that the Alrao model averaging method is standard Bayesian model
averaging. The convexity and fixed classifier assumptions cover, for instance, standard logistic regression: in
that case the Alrao output layer contains copies of a logistic classifier
with various learning rates, and the Alrao pre-classifier is the identity
(or any fixed linear pre-classifier).

For each Alrao classifier $j$, for simplicity we denote its parameters by
$\theta_j$ instead of $\theta_j^{\text{cl}}$ (there is no more ambiguity
since the pre-classifier is fixed).

The loss of some classifier $C$ on a dataset with features $(x_i)$ and labels
$(y_i)$ is $L(C) \deq \frac{1}{N}\sum_i \ell(C(x_i),
y_i)$, where for each input $x_i$, $(C(x_i)_y)_{y\in {\cal Y}}$ is a probability distribution
over the possible labels $y\in {\cal Y}$, and we use the log-loss
$\ell(C(x_i), y_i) \deq  -\log C(x_i)_{y_i}$.

For a classifier $C_\theta$ with parameter $\theta$, let us abbreviate
$L(\theta) \deq L(C_{\theta})$.  We assume that $L(\theta)$ is a
non-negative convex function, with
$\nabla^2L(\theta) \preceq \lambda I$ for all $\theta$.
Let $L^*$ be its global infimum; we assume $L^*$ is a minimum, reached at some point
$\theta^*$, namely $L(\theta^*)
= L^*$. Moreover we assume  that $L$ is locally strongly convex at its
minimum $\theta^*$:
$\nabla^2L(\theta^*) \succ 0$.

The Alrao architecture for such a classifier $C_\theta$ uses
$N_{\text{cl}}$ copies of
the same classifier, with different parameter values:
\begin{equation}
  \Phi^{\alrao}_{\theta_\alrao}(x) = \sum_{j=1}^{N_{\text{cl}}}a_j C_{\theta_j}
\end{equation}
where $\theta_\alrao \deq (\theta_1, ..., \theta_{N_{\text{cl}}})$, and
where the
$(a_j)_j$ are the weights given by the model averaging method. We
abbreviate $L(\theta_\alrao) \deq L(\Phi^{\alrao}_{\theta_\alrao})$.

The Alrao classification layer uses a set of learning rates
$(\eta_j)_{j\in J}$, and starting points $(\theta_j^{(0)})_{j\in J}$.
Using full-batch (non stochastic) Alrao updates we have
\begin{align}
  \theta_j^{(t+1)} &= \theta_j^{(t)} - \eta_j \nabla L(\theta_j^{(t)}) \\
  a^{(t+1)} &=  \mathtt{ModelAveraging}(a^{(t)}, (C_{\theta_{i}}(x_{1:N}))_{i}, y_{1:N})
\end{align}

We assume that the model averaging method is Bayesian Model Averaging.

We have assumed that the Hessian of the loss of the model satisfies
$\nabla^2L(\theta) \preceq \lambda I$. Under this condition, the standard
theory of gradient descent for convex functions requires that the
learning rate be less than $1/\lambda$, otherwise the gradient descent
might diverge. Therefore, for Alrao we assume that at least \emph{one} of
the learning rates considered by Alrao is below this threshold. 

\begin{theorem}
Assume that at least one of the Alrao learning rates $\eta_j$ satisfies
$\eta_j <1/\lambda$, with $\lambda$ as above. Then, under the assumptions
above, the Alrao loss is at most the optimal loss when $t\to\infty$:
  \begin{equation}
    \limsup_t L(\theta_\alrao^{(t)}) \leq L^*
  \end{equation}
\end{theorem}

\begin{proof}
Let us
analyze the dynamics of the different models in the model averaging method.
Let us split the set of Alrao classifiers in two categories according to
whether their sum of errors is finite or infinite, namely,
\begin{align}
A&\deq \left\{j \in J \text{ such that } \sum_{t\geq
0}\left(L(\theta_j^{(t)})-L^*\right) <
\infty\right\},
\\B &\deq \left\{j \in J \text{ such that } \sum_{t\geq
0}\left(L(\theta_j^{(t)})-L^* \right)=
\infty\right\}
\end{align}
and in particular, for any $j\in A$, $\lim_tL(\theta_j^{(t)}) = L^*$.

The proof is organized as follows: We first show that $A$ is not empty.
Then, we show that $\lim_{t \rightarrow
\infty}a_j^{(t)} = 0$ for all $j \in B$: these models are eliminated by
the model averaging method. Then we will be able to conclude.

  \bigskip

First, we show that $A$ is not empty: namely, that there is
least one $j$ such that $\sum_{t\geq 0}(L(\theta_j^{(t)})-L^*) < \infty$.
We know that there is $j$ such that $\eta_j < \frac{2}{\lambda}$. Hence,
the standard theory of gradient descent for convex functions shows that
this particular classifier converges (e.g., \citep{Tibshirani2013}),
namely, the loss
$(L(\theta_j^{(t)}))_t$ converges to $L^*$. Moreover, since
$L$ is localy stricly convex around $\theta^*$, this implies that
$\lim_t \theta_j^{(t)} = \theta^*$.

We now show that the sum of errors for this specific $j$ converges. We
assumed that $L(\theta)$ is locally strongly convex in $\theta^*$. Let
$\mu >0$ such that $\nabla^2L(\theta^*) \succeq \mu I$. Since $L$ is
$C^2$, there is $\varepsilon '$ such that for any $\theta$ such that
$\|\theta - \theta^*\|\leq \varepsilon'$, then $\nabla^2L(\theta) \succeq
\frac{\mu}{2}I$. Let $\tau \in \BN$ such that $\|\theta_j^{(\tau)} -
\theta^*\| < \varepsilon '$. Then, from the theory of gradient descent
for strongly convex functions \citep{Tibshirani2013}, we know there is
some $\gamma < 1$ such that for $t > \tau$, $L(\theta_j^{(t)}) - L^* \leq C \|\theta_j^{(\tau)} - \theta^*\|\gamma^t$. We have:
  \begin{align}
    \sum_{s=1}^t &\left(L(\theta_j^{(s)})-L^*\right) = \\
    &= \sum_{s=1}^\tau
    \left(L(\theta_j^{(s)})-L^*\right) + \sum_{s=\tau}^t
    \left(L(\theta_j^{(s)})-L^*\right)\\
                                       &\leq  \sum_{s=1}^\tau
				       \left(L(\theta_j^{(s)})-L^*\right) + C \|\theta_j^{(\tau)} - \theta^*\|\gamma^\tau\frac{1}{1-\gamma}
  \end{align}
  Thus $\sum_{t \geq 0} \left(L(\theta_j^{(t)})-L^*\right) < \infty$.
  Therefore, $A$ is not empty.

\medskip
We now show that the weights $a_j^{(t)}$ tend to $0$ for any $j \in B$,
namely, $\lim_{t\rightarrow\infty}a_j^{(t)} = 0$.
Let $j\in B$ and take some $i\in A$.  In Bayesian model averaging, the weights are
\begin{align}
    a_j^{(t)} {}&= \frac{\prod_{s=1}^t
    p_{\theta_j^{(s)}}(y_{1:N}|x_{1:N})}{\sum_k \prod_{s=1}^t p_{\theta_k^{(s)}}(y_{1:N}|x_{1:N})}\\
    {}&\leq \prod_{s=1}^t\frac{p_{\theta_j^{(s)}}(y_{1:N}|x_{1:N})}{p_{\theta_i^{(s)}}(y_{1:N}|x_{1:N})}\\
              {}&= \prod_{s=1}^t\frac{\exp(-NL(\theta_j^{(s)}))}{\exp(-NL(\theta_i^{(s)}))}\\
  \begin{split}
    {}&= \exp \Big(- N\sum_{s=1}^t \left(L(\theta_j^{(s)}) - L^*\right) +\\
      & N\sum_{s=1}^t \left(L(\theta_i^{(s)}) - L^*\right)\Big)
  \end{split}
  \end{align}
Since $i\in A$ and $j \in B$, by definition of $A$ and $B$ this tends to
$0$. Therefore,
    $\lim_t a_j^{(t)} = 0$ for all $j\in B$.

\medskip
  We now prove the statement of the theorem. We have:
  \begin{align}
    \begin{split}
    L(\theta_\alrao^{(t)}) ={}& \frac{1}{N}\sum_i -\log \Big(\sum_{j\in
      A}a_je^{-\ell(C_{\theta_j^{(t)}}(x_i), y_i)} +\\
    &\sum_{j\in
    B}a_je^{-\ell(C_{\theta_j^{(t)}}(x_i), y_i)}\Big)  
    \end{split}\\
    {}&\leq \frac{1}{N}\sum_i -\log \left(\sum_{j\in
    A}a_je^{-\ell(C_{\theta_j^{(t)}}(x_i), y_i)}\right)
  \end{align}
  For all $i \in A$, set $\tilde a_i^{(t)} \deq \frac{a_i^{(t)}}{\sum_{j\in A} a_j^{(t)}} = \frac{a_i^{(t)}}{1 - \sum_{j\in B} a_j^{(t)}}$. Then
  \begin{align}
    \begin{split}
      L(\theta_\alrao^{(t)}) \leq{}& -\log\left( 1 - \sum_{j \in B}
      a_j^{(t)}\right) \\ &+ \frac{1}{N}\sum_i -\log \left(\sum_{j\in
      A}\tilde a_je^{-\ell(C_{\theta_j^{(t)}}(x_i), y_i)}\right)
    \end{split}\\
                             &\leq \frac{1}{N}\sum_i \sum_{j\in A}\tilde
			     a_j\ell(C_{\theta_j^{(t)}}(x_i), y_i) + o(1) \\
                             &= \sum_{j\in A}\tilde a_j L(\theta_j^{(t)}) +  o(1)\\
                               &= L^* + o(1)
    \end{align}
    thanks to Jensen's inequality for $-\log$, then because $\lim_t a_j^{(t)} =
    0$ for $j\in B$, and finally because $\lim_t L(\theta_j^{(t)}) = L^*$
    for $j\in A$. Taking the $\limsup$, we have:
    \begin{equation}
      \limsup_t L(\theta_\alrao^{(t)}) \leq L^*
    \end{equation}
which ends the proof.%
\end{proof}

\section{Evolution of the Posterior}
\label{sec:posterior}

The evolution of the model averaging weights can be observed during training. In Figure~\ref{fig:posterior}, we can see their evolution during the training of the GoogLeNet model with Alrao on CIFAR10, 10 classifiers, with $\eta_\min = 10^{-5}$ and $\eta_\max=10^1$.

We can make several observations. First, after only a few gradient
descent steps, the model averaging weights corresponding to the three
classifiers with the largest learning rates go to zero. This means that
their parameters are moving too fast, and their loss is getting very
large.

Next, for a short time, a classifier with a moderately large learning
rate gets the largest posterior weight, presumably because it is the
first to learn a useful model.

Finally, after the model has seen approximately 4,000 samples, a
classifier with a slightly smaller learning rate is assigned a posterior
weight $a_j$ close to 1, while all the others go to 0. This means that
after a number of gradient steps, the model averaging method acts like a
model selection method.

\begin{figure*}[h]
\centering
\includegraphics[width=\linewidth]{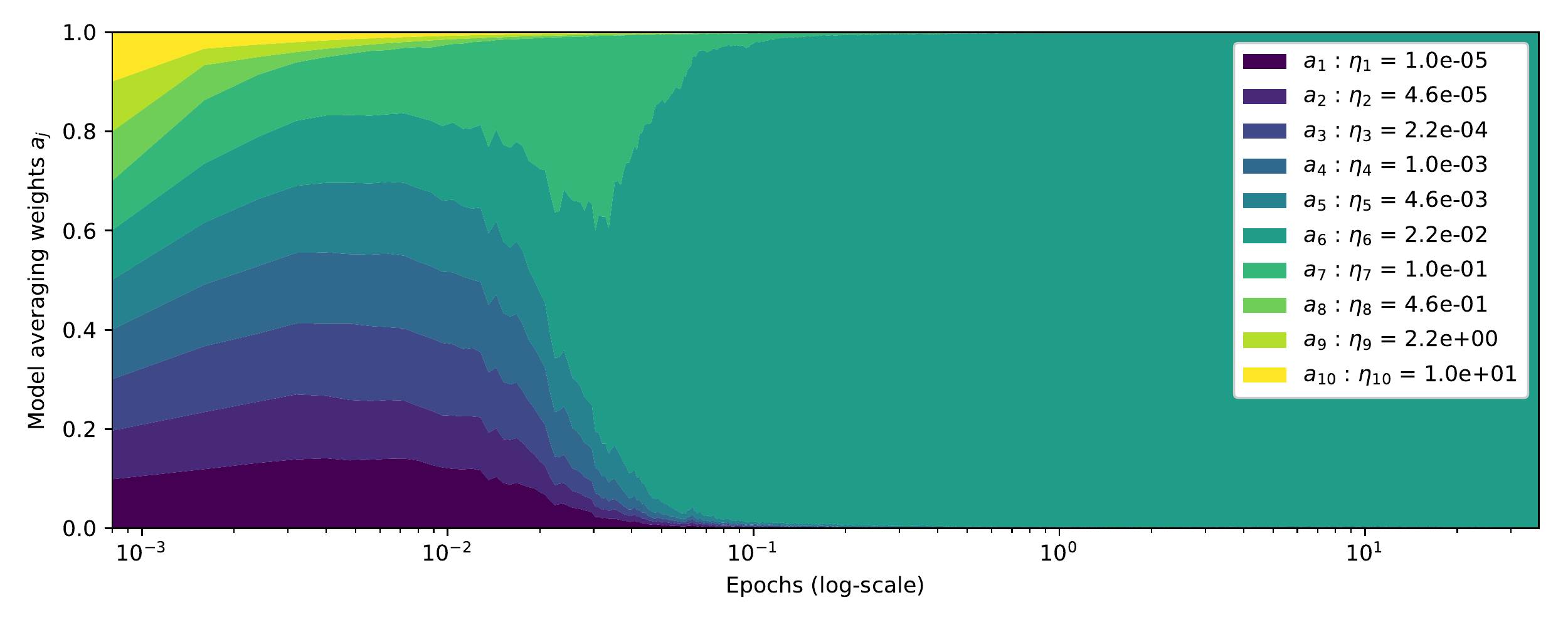}
\caption{Model averaging weights during training. During the training of the GoogLeNet model with Alrao on CIFAR10, 10 classifiers, with $\eta_\min = 10^{-5}$ and $\eta_\max=10^1$, we represent the evolution of the model averaging weights $a_j$, depending on the corresponding classifier's learning rate. }
\label{fig:posterior}
\end{figure*}

\section{Additional Experimental Details and Results}
\label{sec:preprocess}

In the case of CIFAR-10 and ImageNet, we normalize each input channel $x_i$ ($1 \leq i \leq 3$), using its mean and its standard deviation over the training set. Let $\mu_i$ and $\sigma_i$ be respectively the mean and the standard deviation of the $i$-th channel. Then each input $(x_1, x_2, x_3)$ is transformed into $(\frac{x_1 - \mu_1}{\sigma_1}, \frac{x_2 - \mu_2}{\sigma_2}, \frac{x_3 - \mu_3}{\sigma_3})$. This operation is done over all the data (training, validation and test).

Moreover, we use data augmentation: every time an image of the training set is sent as input of the NN, this image is randomly cropped and and randomly flipped horizontally. Cropping consists in filling with black a band at the top, bottom, left and right of the image. The size of this band is randomly chosen between 0 and 4 in our experiments.

On CIFAR10 and PTB, the batch size was 32 for every architecture. On ImageNet, the batch-size is 256 for Alexnet and ResNet50, and 128 for Densenet121.

On the Reinforcement Learning environments, we used vanilla Q-learning \cite{dqn} with a soft target update as in \cite{ddpg} $\tau = 0.9$, and a memory buffer of size 1,000,000. The architecture for the Q network is a MLP with $2$ hidden layers. The learning curves are in Fig. \ref{fig:rl-learningcurves}. For the optimisation, the switch was used with 10 output layers. An output layer is a linear layer. Since the switch is a probability model averaging method, we consider each output layer as a probabilistic model, defined as a Normal distribution with variance 1 and mean the predicted value by the output layer. The loss for the Alrao model is the negative log-likelihood of the model mixture.

\begin{figure*}[p]
  \centering
  \subfloat[GoogLeNet on CIFAR10 (Average on three runs)  \label{fig:firstepochs-googlenet}]{
    \includegraphics[width=\linewidth]{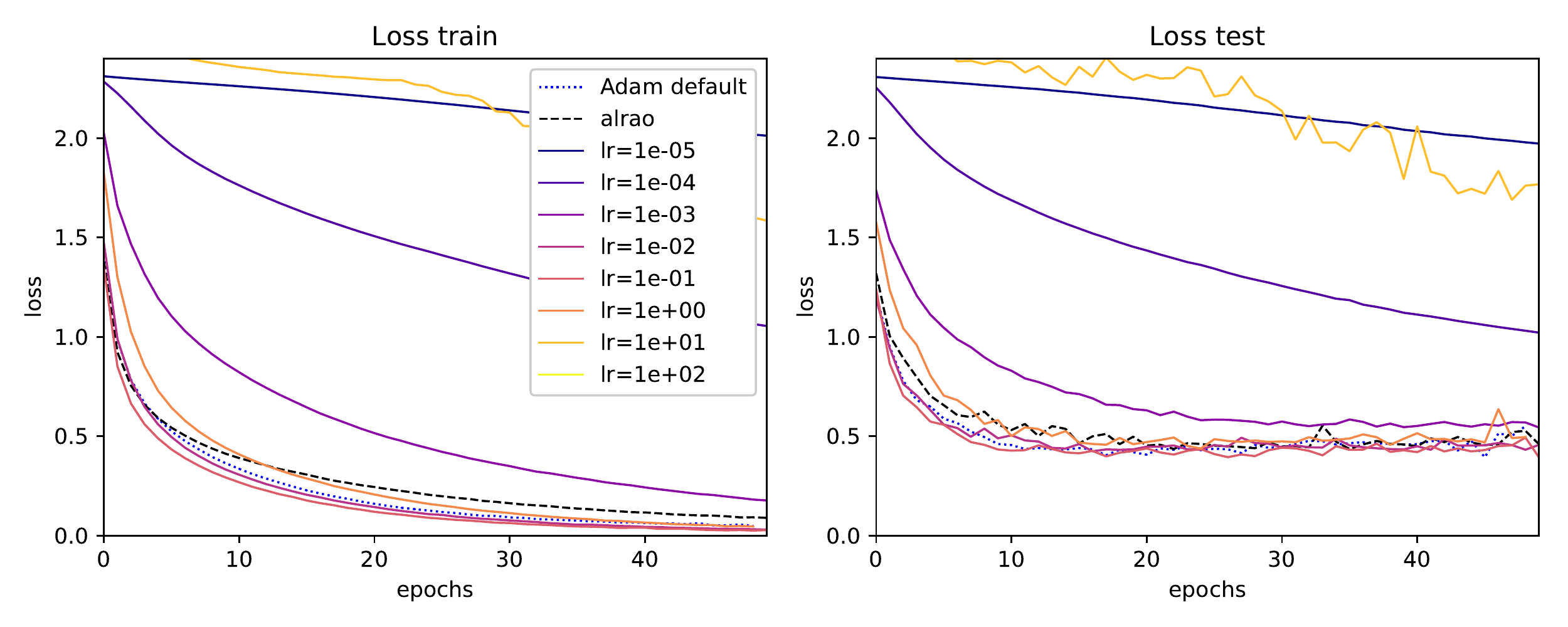}
    }
\vfill
\subfloat[Densenet121 trained on ImageNet  \label{fig:firstepochs-densenet}]{
  \includegraphics[width=\linewidth]{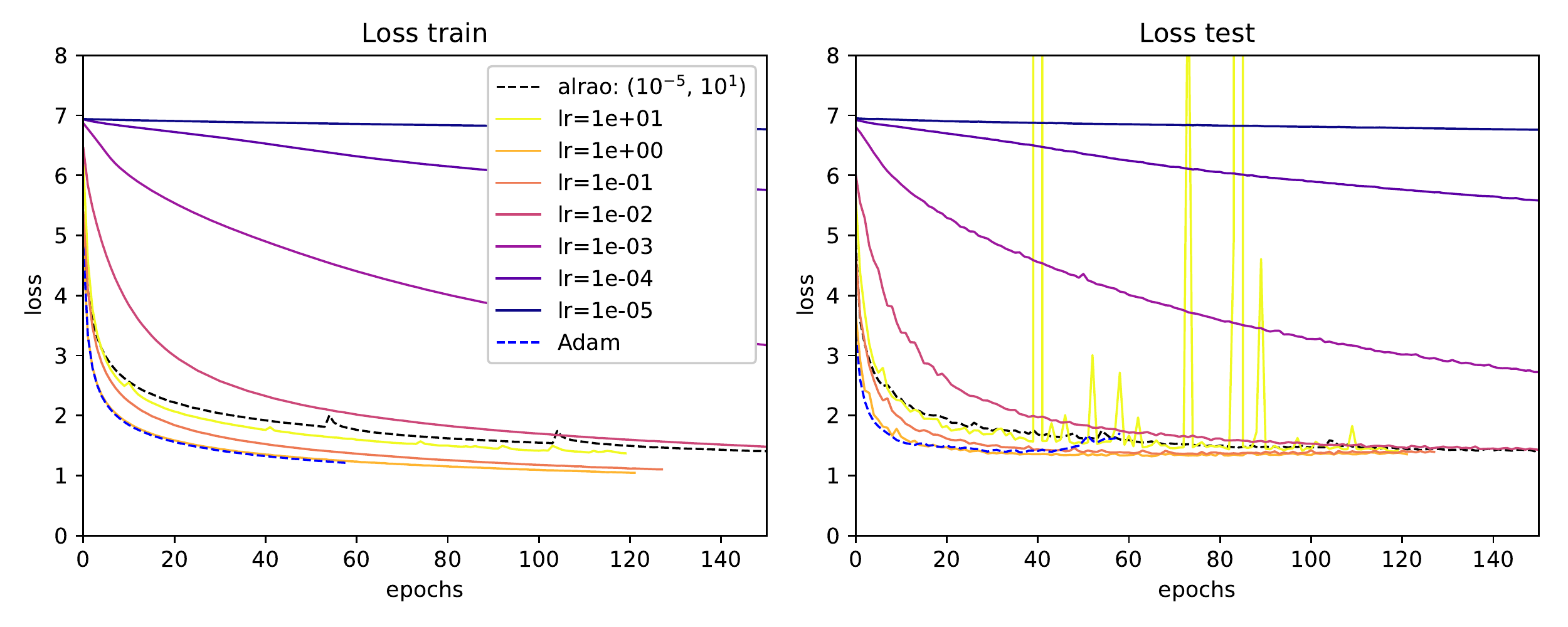}
  }
\vfill
\subfloat[MobileNetV2 on Cifar10 (average over 3 runs)\label{fig:firstepochs-mobilenet}]{
        \includegraphics[width=0.95\linewidth]{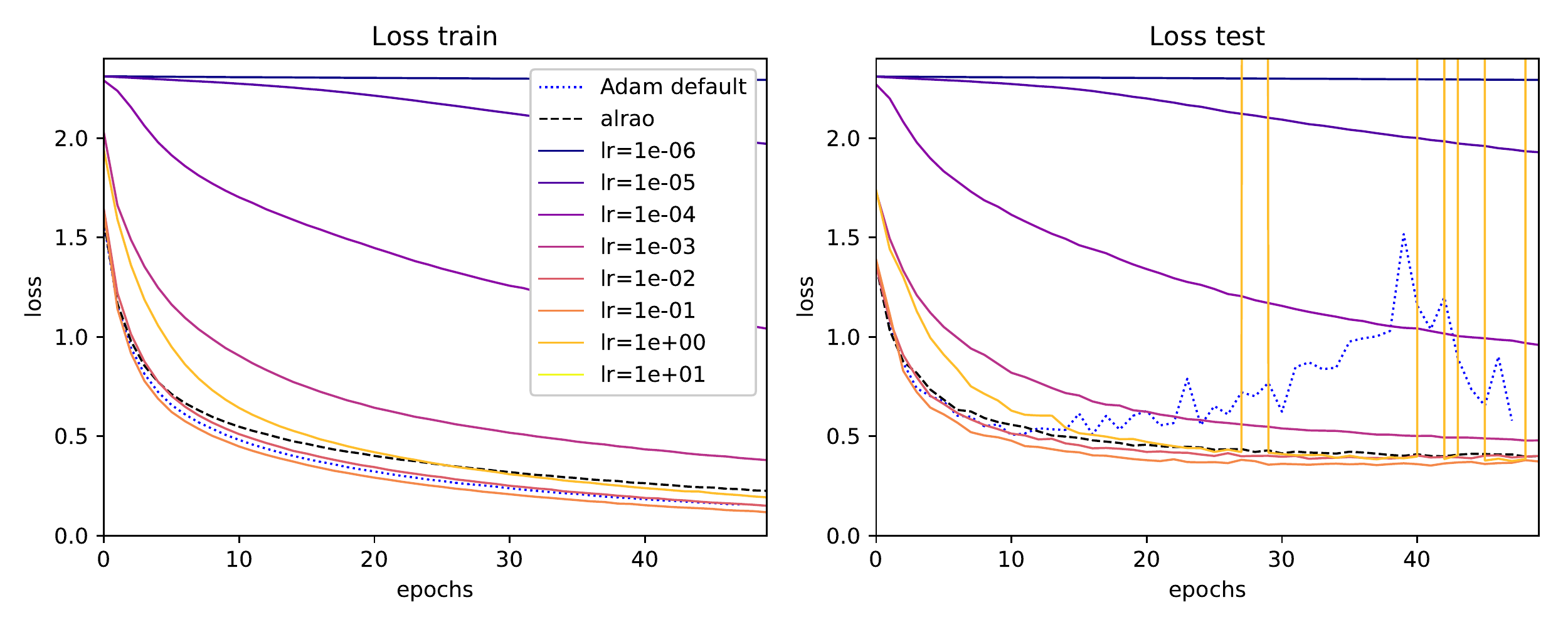}
      }
  \caption{Additional learning curves for SGD with various learning rates, Alrao, and Adam with its default setting, with the Densenet121 and Alexnet architectures on ImageNet and the GoogLeNet architecture on CIFAR10. Left: training loss; right: test loss.}
  \label{fig:adam-alrao}
\end{figure*}

\begin{figure}[h]
  \centering
  \includegraphics[width=\linewidth]{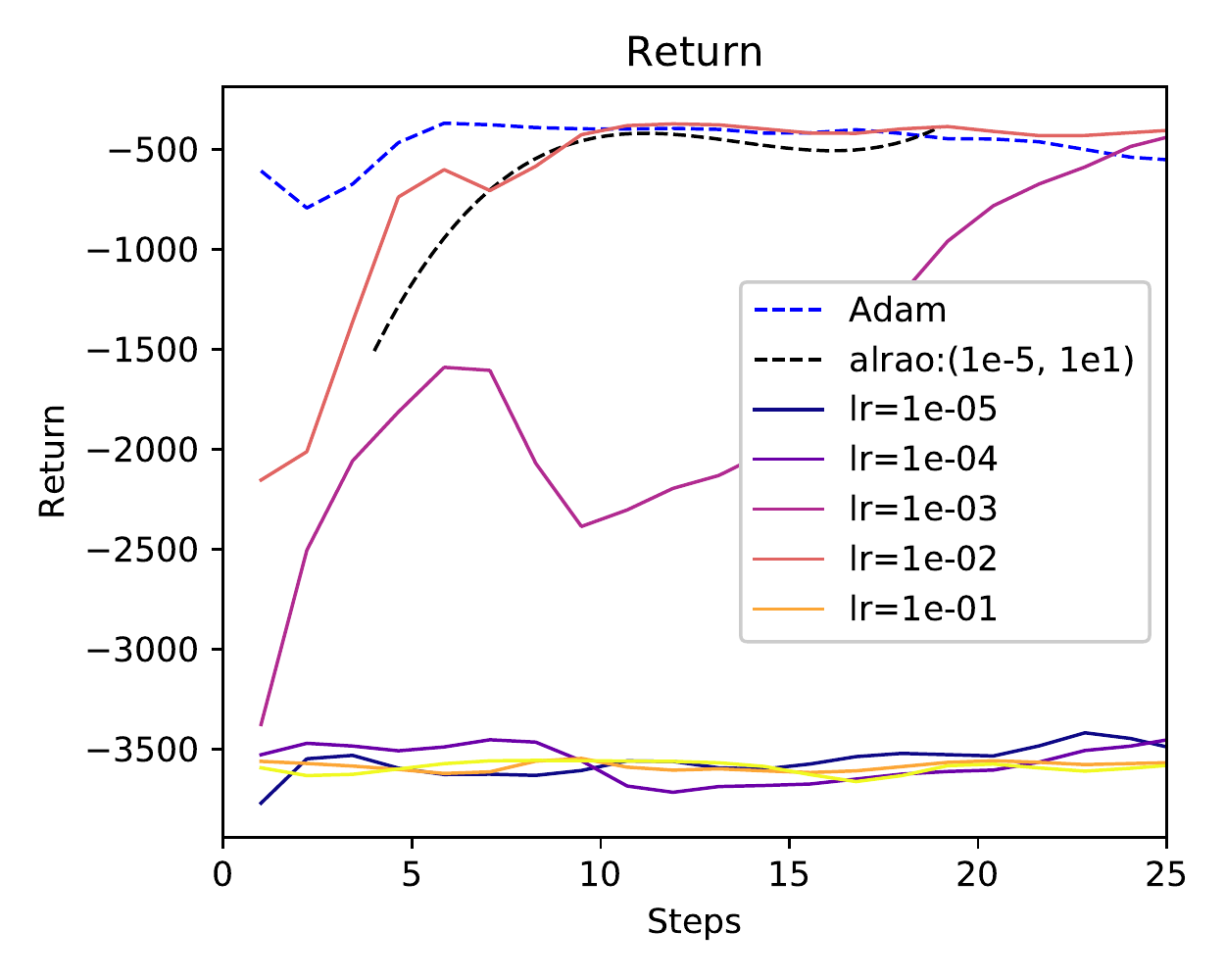}
  \caption{Learning curves for the Reinforcement Learning environment in the pendulum environment in Q-learning, for SGD with various learning rates, Alrao, and Adam with its default setting}
  \label{fig:rl-learningcurves}
\end{figure}

\section{Alrao with Adam}
\label{sec:alrao-adam}

In Figure~\ref{fig:adam-alrao}, we report our experiments with
Alrao-Adam on CIFAR10. As explained in Section~\ref{sec:strengths-weaknesses}, Alrao
is much less reliable with Adam than with SGD.

This is especially true for the test
performance, which can even diverge while training performance remains
either good or acceptable (Fig.~\ref{fig:adam-alrao}). Thus Alrao-Adam seems to send the model
into atypical regions of the search space.

We have no definitive explanation for this at present.
It might be that
changing Adam's learning rate requires changing its momentum parameters
accordingly.  It might be that Alrao does not work on Adam because Adam
is more sensitive to its hyperparameters. 


\begin{figure*}[p]
  \centering
    \subfloat[Alrao-Adam with GoogLeNet on CIFAR10: Alrao-Adam compared with standard Adam with various learning rates. Alrao uses 10 classifiers and learning rates in the interval $(10^{-6}, 1)$. Each plot is averaged on 10 experiments. We observe that optimization with Alrao-Adam is efficient, since train loss is comparable to the usual Adam methods. But the model starkly overfits, as the test loss diverges. \label{fig:googlenet-adamalrao}]{
    \includegraphics[width=0.8\textwidth]{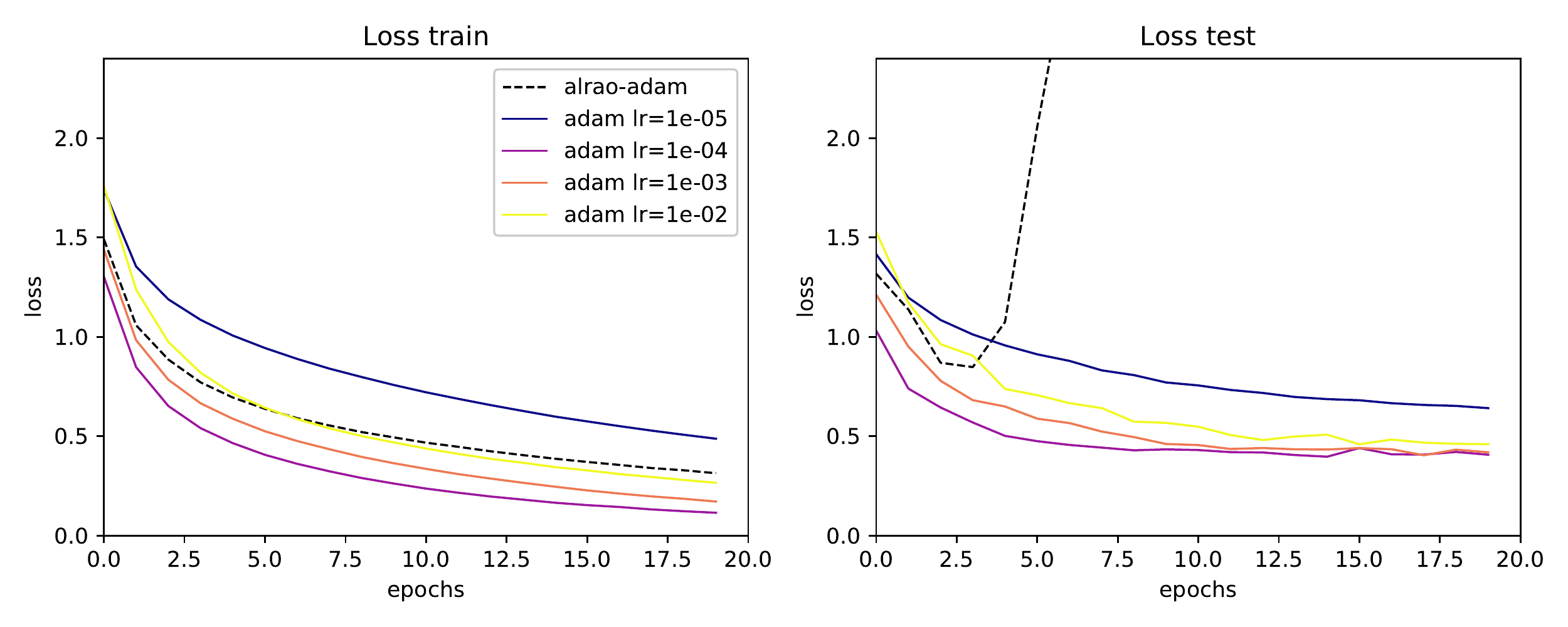}
  }
\vfill
\subfloat[Alrao-Adam with MobileNet on CIFAR10: Alrao-Adam with two different learning rate intervals, $(10^{-6}, 10^{-2})$ for the first one, $(10^{-6}, 10^{-1})$ for the second one, with 10 classifiers each. The first one is with $\eta_\min = 10^{-6}$. Each plot is averaged on 10 experiments. Exactly as with GoogLeNet model, optimization itself is efficient (for both intervals). For the interval with the smallest $\eta_\max$, the test loss does not converge and is very unstable. For the interval with the largest $\eta_\max$, the test loss diverges.\label{fig:mobilenet-adamalrao}]{
  \includegraphics[width=.8\textwidth]{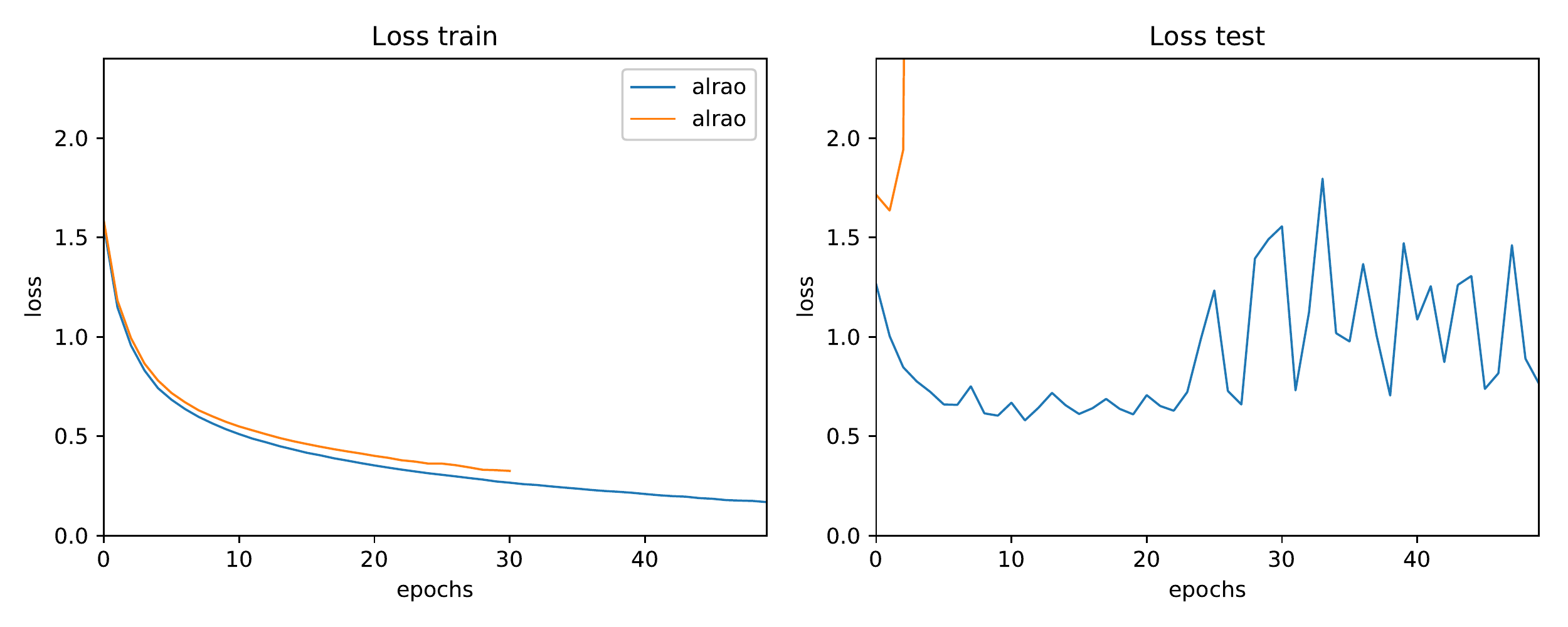}
}
  \vfill
  \subfloat[Alrao-Adam with VGG19 on CIFAR10:  Alrao-Adam on the interval $(10^{-6}, 1)$, with 10 classifiers. The 10 plots are 10 runs of the same experiments. While 9 of them do converge and generalize, the last one exhibits wide oscillations, both in train and test.  \label{fig:vgg-adamalrao}]{
    \includegraphics[width=.8\textwidth]{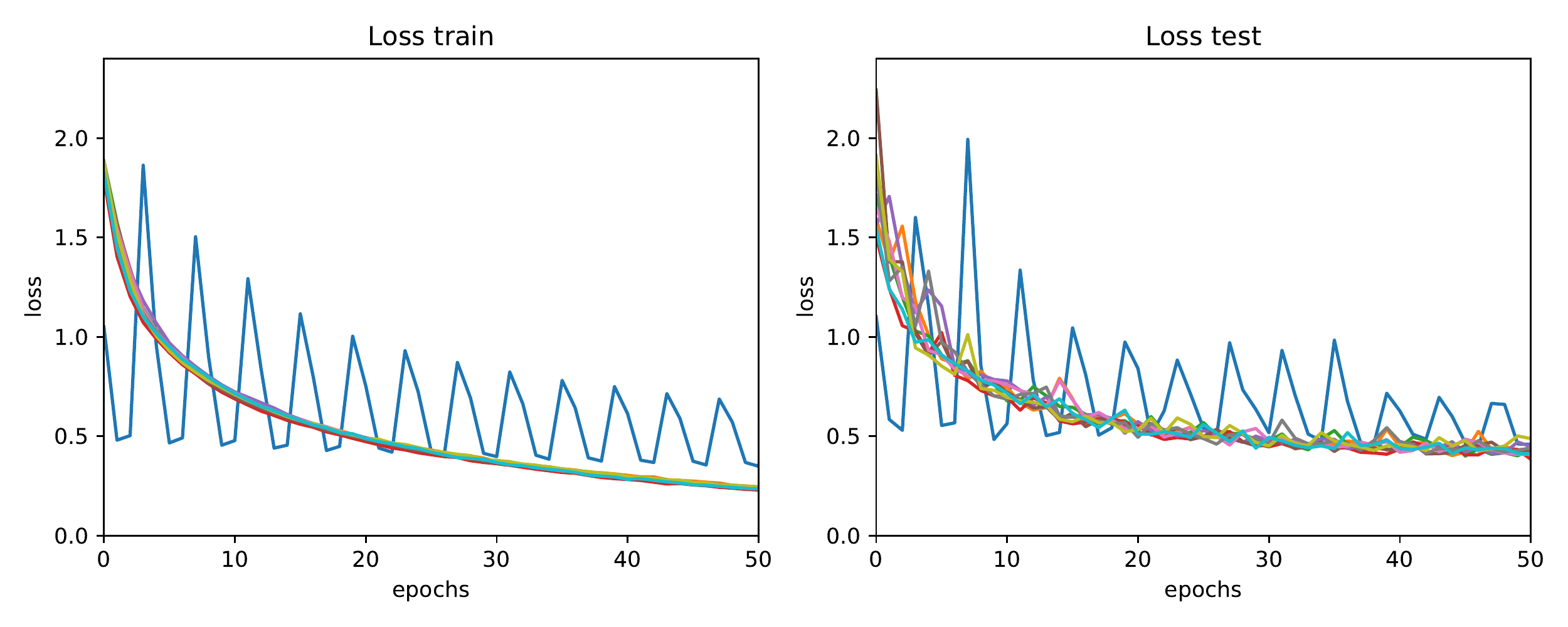}
    }
  \caption{Alrao-Adam: Experiments with the VGG19, GoogLeNet and MobileNet networks on CIFAR10.}
  \label{fig:adam-alrao}
\end{figure*}





\section{Number of Parameters}
\label{sec:number-parameters}

As explained in Section~\ref{sec:strengths-weaknesses}, Alrao increases the number of
parameters of a model, due to output layer copies. The additional number
of parameters is approximately equal to $(N_{\mathrm{cl}} - 1)\times K
\times d$ where $N_{\mathrm{cl}}$ is the number of classifier copies used in
Alrao, $d$ is the dimension of the output of the pre-classifier, and $K$
is the number of classes in the classification task (assuming a standard
softmax output; classification with many classes often uses other kinds
of output
parameterization instead).

\begin{table}[!h]
  \caption{Comparison between the number of parameters in models used without and with Alrao. LSTM (C) is a simple LSTM cell used for character prediction while LSTM (W) is the same cell used for word prediction.}
  \begin{center}
  \fontsize{9pt}{9pt}\selectfont
  \begin{sc}
  \begin{tabular}[h]{lcc}
    \toprule
    Model     & \multicolumn{2}{c}{Number of parameters} \\
    {}        & Without Alrao & With Alrao                \\
    \midrule
    GoogLeNet & 6.166M       & 6.258M                    \\
    VGG       & 20.041M      & 20.087M                   \\
    MobileNet & 2.297M       & 2.412M                    \\
    \midrule
    LSTM (C)  & 0.172M       & 0.197M                    \\
    LSTM (W)  & 2.171M       & 7.221M                    \\
    \bottomrule
  \end{tabular}
  \end{sc}
  \end{center}

  \vskip -0.1in
  \label{tab:nparams}
\end{table}

The number of parameters for the models used, with and without Alrao, are
in Table~\ref{tab:nparams}. We used 10 classifiers in Alrao for
convolutional neural networks, and 6 classifiers for LSTMs. Using Alrao for classification tasks with many classes,
such as word prediction (10,000 classes on PTB), increases the number of
parameters noticeably.

For those model with significant parameter increase, the various
classifier copies may be done on parallel GPUs.

\section{Frozen Features Do Not Hurt Training}
\label{sec:alrao-bernouilli}

\begin{figure*}[h]
  \centering
  \includegraphics[width=0.7\linewidth]{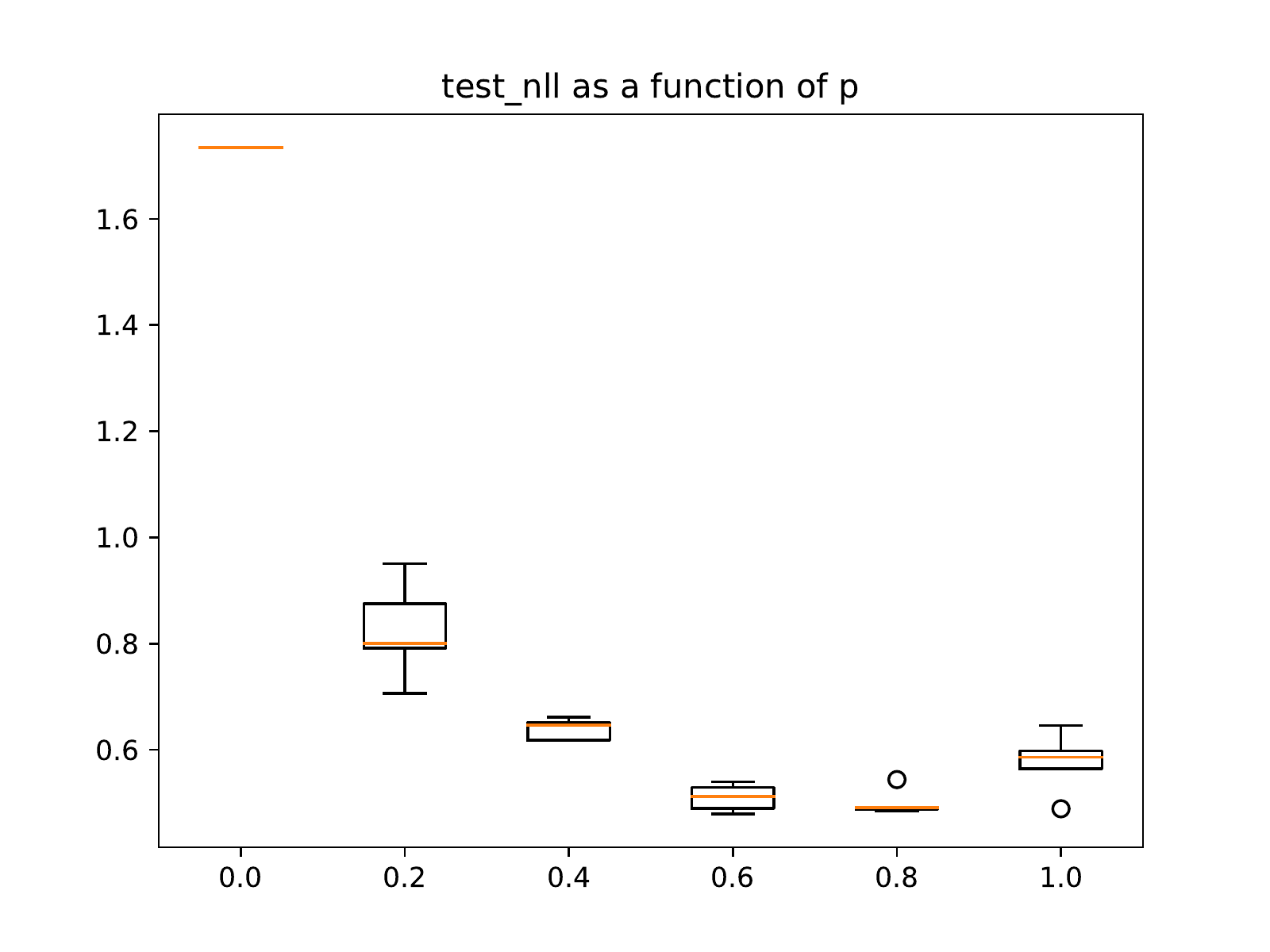}
  \caption{Loss of a model where only a random fraction $p$ of the features are trained, and the others left at their initial value, as a function of $p$. The architecture is GoogLeNet, trained on CIFAR10.}
  \label{fig:fraction-feature}
\end{figure*}

As explained in the introduction, several works support the idea that not
all units are useful when learning a deep learning model. Additional
results supporting this hypothesis are presented in
Figure~\ref{fig:fraction-feature}. We trained a GoogLeNet architecture on
CIFAR10 with standard SGD with learning rate $\eta_0$, but learned only a
random fraction $p$ of the features (chosen at startup), and kept the
others at their initial value. This is equivalent to sampling each learning
rate $\eta$ from the probability distribution $P(\eta = \eta_0) = p$ and
$P(\eta = 0) = 1 - p$.

We observe that even with a fraction of the weights not being learned,
the model's performance is close to its performance when fully trained.

When training a model with Alrao, many features might not learn at all,
due to too small learning rates. But Alrao is still able to reach good
results.  This could be explained by the resilience of
neural networks to partial training.

\section{Tutorial}
\label{sec:tutorial}

In this section, we briefly show how Alrao can be used in practice on an
already implemented method in Pytorch. The full code will be available
once the anonymity constraint is lifted.

The first step is to build the preclassifier. Here, we use the VGG19 architecture. The model is built without a classifier. Nothing else is required for Alrao at this step.

\begin{lstlisting}[language=python]
class VGG(nn.Module):
    def __init__(self, cfg):
        super(VGG, self).__init__()
        self.features = self._make_layers(cfg)
        # The dimension of the preclassier's output need to be specified.
        self.linearinputdim = 512

    def forward(self, x):
        out = self.features(x)
        out = out.view(out.size(0), -1)
        # The model do not contain a classifier layer.
        return out

    def _make_layers(self, cfg):
        layers = []
        in_channels = 3
        for x in cfg:
            if x == 'M':
                layers += [nn.MaxPool2d(kernel_size=2, stride=2)]
            else:
                layers += [nn.Conv2d(in_channels, x, kernel_size=3, padding=1),
                           nn.BatchNorm2d(x),
                           nn.ReLU(inplace=True)]
                in_channels = x
        layers += [nn.AvgPool2d(kernel_size=1, stride=1)]
        return nn.Sequential(*layers)

preclassifier = VGG([64, 64, 'M', 128, 128, 'M', 256, 256, 256, 256, 'M', \
        512, 512, 512, 512, 'M', 512, 512, 512, 512, 'M'])

\end{lstlisting}

Then, we can build the Alrao-model with this preclassifier, sample the learning rates for the model, and define the Alrao optimizer

\begin{lstlisting}[language=python]
# We define the interval in which the learning rates are sampled
minlr = 10 ** (-5)
maxlr = 10 ** 1

# nb_classifiers is the number of classifiers averaged by Alrao.
nb_classifiers = 10
nb_categories = 10

net = AlraoModel(preclassifier, nb_categories, preclassifier.linearinputdim, nb_classifiers)

# We spread the classifiers learning rates log-uniformly on the interval.
classifiers_lr = [np.exp(np.log(minlr) + \
    k /(nb_classifiers-1) * (np.log(maxlr) - np.log(minlr)) \
    ) for k in range(nb_classifiers)]

# We define the sampler for the preclassifier's features.
lr_sampler = lr_sampler_generic(minlr, maxlr)
lr_preclassifier = generator_randomlr_neurons(net.preclassifier, lr_sampler)

# We define the optimizer
optimizer = SGDAlrao(net.parameters_preclassifier(),
                      lr_preclassifier,
                      net.classifiers_parameters_list(),
                      classifiers_lr)
\end{lstlisting}

Finally, we can train the model. The only differences here with the usual training procedure is that each classifier needs to be updated as if it was alone, and that we need to update the model averaging weights, here the switch weights.

\begin{lstlisting}[language=python]
def train(epoch):
    for batch_idx, (inputs, targets) in enumerate(trainloader):
        # We update the model averaging weights in the optimizer
        optimizer.update_posterior(net.posterior())
        optimizer.zero_grad()

        # Forward pass of the Alrao model
        outputs = net(inputs)
        loss = nn.NLLLoss(outputs, targets)

        # We compute the gradient of all the model's weights
        loss.backward()

        # We reset all the classifiers gradients, and re-compute them with
        # as if their were the only output of the network.
        optimizer.classifiers_zero_grad()
        newx = net.last_x.detach()
        for classifier in net.classifiers():
            loss_classifier = criterion(classifier(newx), targets)
            loss_classifier.backward()

        # Then, we can run an update step of the gradient descent.
        optimizer.step()

        # Finally, we update the model averaging weights
        net.update_switch(targets, catch_up=False)
\end{lstlisting}

\end{document}